\documentclass[11pt]{article}

\date{}

\usepackage{siunitx}
\usepackage{outlines}
\usepackage{times}
\usepackage{amsmath}
\usepackage{amssymb}
\usepackage{amsthm}
\usepackage{bm}
\usepackage{hyperref}
\usepackage{cleveref}
\usepackage{amsthm}
\usepackage{color}
\usepackage{graphicx}
\usepackage{caption}
\usepackage{subcaption}
\usepackage[ruled,noend,linesnumbered]{algorithm2e} 
\usepackage{nameref}
\usepackage{outlines}
\usepackage{multirow}
\usepackage{float}
\usepackage[normalem]{ulem}
\usepackage[square,numbers]{natbib}
\usepackage{slashbox}
\usepackage{sectsty}

\usepackage{subfiles}

\crefname{figure}{Fig.}{Fig.}
\crefname{table}{Table}{Tables}
\crefname{equation}{Eq.}{Eq.}
\Crefname{figure}{Figure}{Figures}

\newtheorem{theorem}{Theorem}
\newtheorem{lemma}[theorem]{Lemma}
\newtheorem{corollary}[theorem]{Corollary}

\crefname{theorem}{Theorem}{Theorem}
\crefname{corollary}{Corollary}{Corollary}
\crefname{fact}{Fact}{Fact}
\crefname{lemma}{Lemma}{Lemma}

\newcommand{\B}[1]{\mathbb{#1}}
\newcommand{\BB}[0]{}

\newcommand{\ahat}{\hat{\mathbf{a}}}

\newcommand{\lyap}{\mathcal{V}}
\newcommand{\atilde}{\tilde{\BB a}}
\newcommand{\stackedstate}{\begin{bmatrix} \BB s \\ \atilde \end{bmatrix}}
\newcommand{\stackedstatedot}{\begin{bmatrix} \dot{\BB s} \\ \dot{\atilde} \end{bmatrix}}
\newcommand{\expo}[1]{\mathrm{e}^{#1}}
\newcommand{\lammin}[0]{\lambda_\text{\normalfont{min}}}
\newcommand{\lammax}[0]{\lambda_\text{\normalfont{max}}}

\newcommand{\revision}[1]{{#1}}
\newcommand{\rev}[1]{{#1}}
\newcommand{\revB}[1]{#1}
\newcommand{\revC}[1]{#1}

\newcommand{\LOne}{{\mathcal{L}_1}}

\newcommand{\kth}[0]{$k^{\mathrm{th}}$ }

\newcommand{\inv}[0]{^{-1}}
\newcommand{\metric}[0]{\mathcal{M}}
\newcommand{\unf}[0]{u_\text{NF}}

\newcommand{\nftransfer}[0]{Neural-Fly-Transfer}
\newcommand{\nf}[0]{Neural-Fly}
\newcommand{\nfconstant}[0]{Neural-Fly-Constant}
\newcommand{\LearningAlg}[0]{Domain Adversarially Invariant Meta-Learning}

\newcommand{\DAML}[0]{DAIML}
\newcommand{\nowind}{\SI{0}{km/h}}
\newcommand{\lowwind}{\SI{15.1}{km/h}}
\newcommand{\medwind}{\SI{30.6}{km/h}}
\newcommand{\highwind}{\SI{43.6}{km/h}}
\newcommand{\sinwind}{$30.6+8.6\sin(t)$ \SI{}{km/h}}

\newcommand{\lowwindalt}{\SI{4.2}{m/s}}
\newcommand{\medwindalt}{\SI{8.5}{m/s}}
\newcommand{\highwindalt}{\SI{12.1}{m/s}}
\newcommand{\sinwindalt}{$8.5+2.4\sin(t)$ \SI{}{m/s}}
\newcommand{\sota}{{state-of-the-art}}

\title{Neural-Fly \revB{Enables} Rapid Learning \\ for Agile Flight in Strong Winds}

\author
{Michael O'Connell$^{\dagger}$, Guanya Shi$^{\dagger}$, Xichen Shi, \\ Kamyar Azizzadenesheli, Anima Anandkumar, Yisong Yue, Soon-Jo Chung$^{*}$
\\
\\
\normalsize{Division of Engineering and Applied Science, California Institute of Technology} \\
\normalsize{$^\dagger$ The first two authors contributed equally to this article. Alphabetical order.} \\
\normalsize{$^*$ Corresponding authors. Email: sjchung@caltech.edu} \\
\\
\normalsize{\textcolor{blue}{This is the accepted version of Science Robotics Vol. 7, Issue 66, eabm6597 (2022)}}\\
\normalsize{DOI: 10.1126/scirobotics.abm6597 \quad Video: \url{https://youtu.be/TuF9teCZX0U}} \\
\normalsize{Data and training code: \url{https://github.com/aerorobotics/neural-fly}}
}

\begin{document} 


\maketitle 
\sectionfont{\MakeUppercase}

\begin{abstract}
Executing safe and precise flight maneuvers in dynamic high-speed winds is \revB{important} for the ongoing commoditization of uninhabited aerial vehicles (UAVs). However, since the relationship between various wind conditions and its \revB{effect} on aircraft maneuverability is not well understood, it is challenging to design effective robot controllers using traditional control design methods. We present \nf, a learning-based approach that allows rapid online adaptation by incorporating pre-trained representations through deep learning. \nf~builds on two key observations that aerodynamics in different wind conditions share a common representation and that the wind-specific part lies in a low-dimensional space. To that end, \nf~uses a \revB{proposed} learning algorithm, \LearningAlg~(\DAML), to learn the shared representation, only using 12 minutes of flight data. With the learned representation as a basis, \nf~then uses a composite adaptation law to update a set of linear coefficients for mixing the basis elements. When evaluated under challenging wind conditions generated with the Caltech Real Weather Wind Tunnel with wind speeds up to \rev{\highwind~(\highwindalt)}, \nf~achieves precise flight control with \revB{substantially} smaller tracking error than state-of-the-art nonlinear and adaptive controllers. In addition to strong empirical performance, the exponential stability of \nf~results in robustness guarantees.
Finally, our control design \rev{extrapolates to unseen wind conditions, is shown to be effective for outdoor flights with only on-board sensors}, and can transfer across drones with minimal performance degradation.
\end{abstract}


\section{Introduction}
\begin{figure}[htbp]
    \centering
    \includegraphics[width=1.0\linewidth]{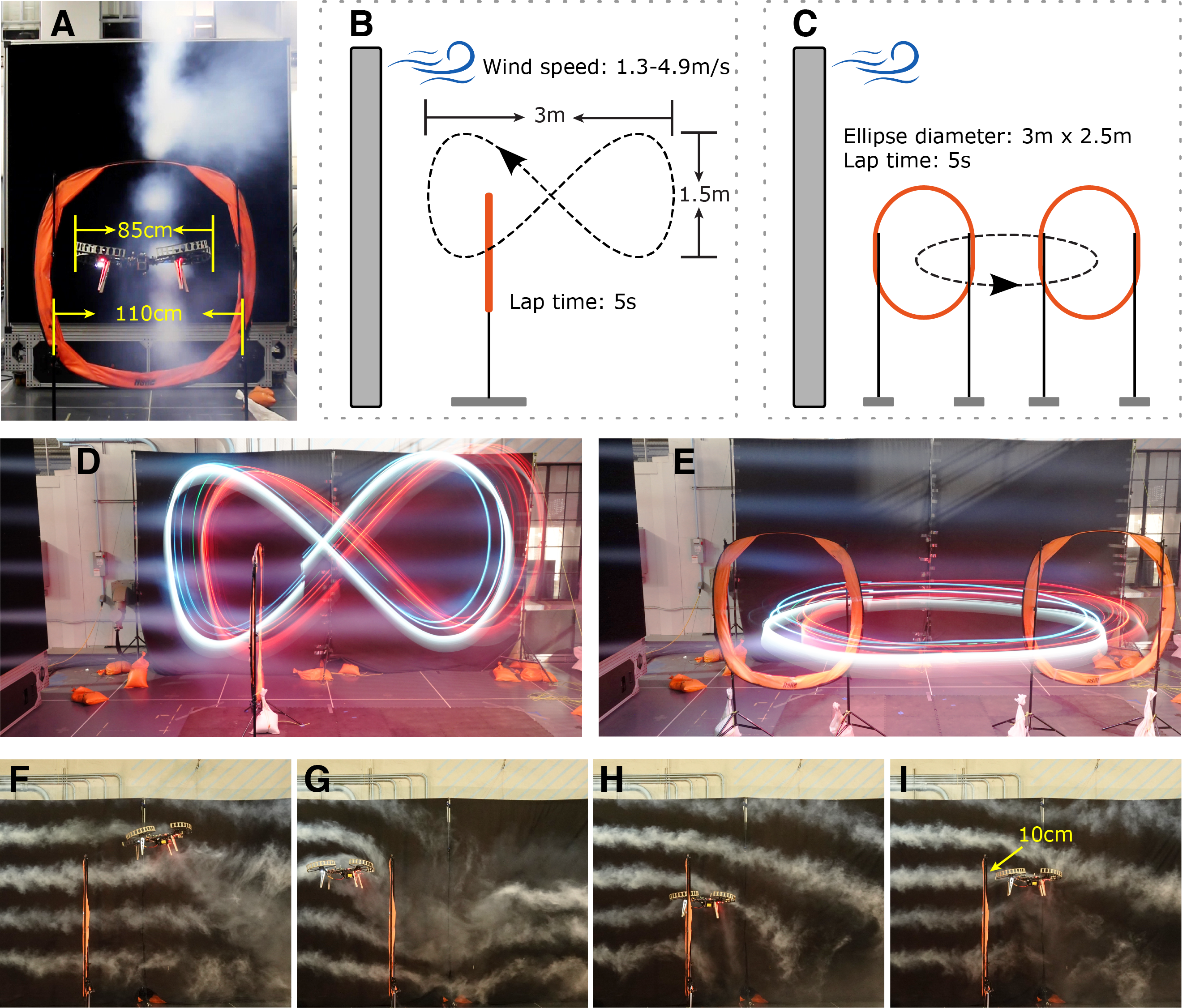}
    \caption{\textbf{Agile flight through narrow gates.} (\textbf{A}) Caltech Real Weather Wind Tunnel system, the quadrotor UAV, and the gate. In our flight tests, the UAV follows an agile trajectory through narrow gates, which are slightly wider than the UAV itself, under challenging wind conditions. (\textbf{B-C}) Trajectories used for the gate tests. In (B), the UAV follows a figure-8 through one gate, with wind speed \SI{3.1}{m/s} or time-varying wind condition. In (C), the UAV follows an ellipse in the horizontal plane through two gates, with wind speed \SI{3.1}{m/s}. (\textbf{D-E}) Long-exposure photos (with an exposure time of \SI{5}{s}) showing one lap in two tasks. (\textbf{F-I}) High-speed photos (with a shutter speed of 1/200\SI{}{s}) showing the moment the UAV passed through the gate and the interaction between the UAV and the wind.}
    \label{fig:intro}
\end{figure}

The commoditization of uninhabited aerial vehicles (UAVs) \revB{requires that the control of these vehicles become more precise and agile}. For example, drone delivery requires transporting goods to a narrow target area in various weather conditions; drone rescue and search require entering and searching collapsed buildings with little space; urban air mobility needs a flying car to follow a planned trajectory closely to avoid collision in the presence of strong unpredictable winds.

Unmodeled and often complex aerodynamics are among the most \revB{notable} challenges to precise flight control. 
Flying in windy environments (as shown in \cref{fig:intro}) introduces even more complexity because of the unsteady aerodynamic interactions between the drone, the induced airflow, and the wind (see \cref{fig:intro}(F) for a smoke visualization). These unsteady and nonlinear aerodynamic effects substantially degrade the performance of conventional UAV control methods that neglect to account for them in the control design. Prior approaches partially capture these effects with simple linear or quadratic air drag models, which limit the tracking performance in agile flight and cannot be extended to external wind conditions \cite{foehn2021time,faessler_differential_2018}. 
Although more complex aerodynamic models can be derived from computational fluid dynamics \cite{ventura2018high}, such modelling is often computationally expensive, and is limited to steady non-dynamic wind conditions. \rev{Adaptive control addresses this problem by estimating linear parametric uncertainty in the dynamical model in real time to improve tracking performance. Recent \sota~in quadrotor flight control has used adaptive control methods that directly 
estimate the unknown aerodynamic force without assuming the structure of the underlying physics, but relying on high-frequency and low-latency control \cite{tal_accurate_2021,mallikarjunan_l1_2012,pravitra_L1-adaptive_2020,hanover_performance_2021}.}
In parallel, there has been increased interest in data-driven modeling of aerodynamics (e.g., \cite{shi2019neural,shi2020neural,shi2021neural,torrente2021data}),
however existing approaches cannot effectively adapt in changing or unknown environments such as time-varying wind conditions.

In this \revB{article}, we present a data-driven approach called \nf, which is a deep-learning-based trajectory tracking controller that learns to quickly adapt to rapidly-changing wind conditions. 
Our method, depicted in \cref{fig:block-diagram}, advances and offers insights into both adaptive flight control and deep-learning-based robot control.
Our experimental demonstrates that \nf~achieves centimeter-level position-error tracking of an agile and challenging trajectory in dynamic wind conditions on a \rev{standard} UAV.

Our method has two main components: an offline learning phase and an online adaptive control phase used as real-time online learning. For the offline learning phase, we have developed \LearningAlg~(\DAML) that learns a wind-condition-independent deep neural network (DNN) representation of the aerodynamics in a data-efficient manner. 
The output of the DNN is treated as a set of basis functions that represent the aerodynamic effects. This representation is adapted to different wind conditions by updating a set of linear coefficients that mix the output of the DNN. 
\DAML~\rev{is data efficient and} uses only 12 total minutes of flight data in \rev{just} 6 different wind conditions to train the DNN. 
\DAML~incorporates several key features which not only improve the data efficiency but also are informed by the downstream online adaptive control phase. In particular, \DAML~uses spectral normalization \cite{shi2019neural,bartlett2017spectrally} to control the Lipschitz property of the DNN to improve generalization to unseen data and \rev{provide closed-loop stability and robustness guarantees.} \DAML~also uses a discriminative network, which ensures that the learned representation is wind-\rev{invariant} and that the wind-dependent information is only contained in the linear coefficients that are adapted in the online control phase.

For the online adaptive control phase, we have developed a regularized composite adaptive control law, which \rev{we} derived from a fundamental understanding of how the learned representation interacts with the closed-loop control system and which \rev{we} support \rev{with} rigorous theory. The use of composite adaptive control was inspired by~\cite{slotine1991applied,slotine_composite_1989}. The adaptation law updates the wind-dependent linear coefficients using a composite of the position tracking error term and the aerodynamic force prediction error term. Such a principled approach effectively guarantees stable and fast adaptation to any wind condition and robustness against imperfect learning. \rev{Although this adaptive control law could be used with a number of learned models,} the speed of adaptation is further aided by the concise representation learned from \DAML.

\begin{figure}[htbp]
    \centering
    \includegraphics[width=1.0\linewidth]{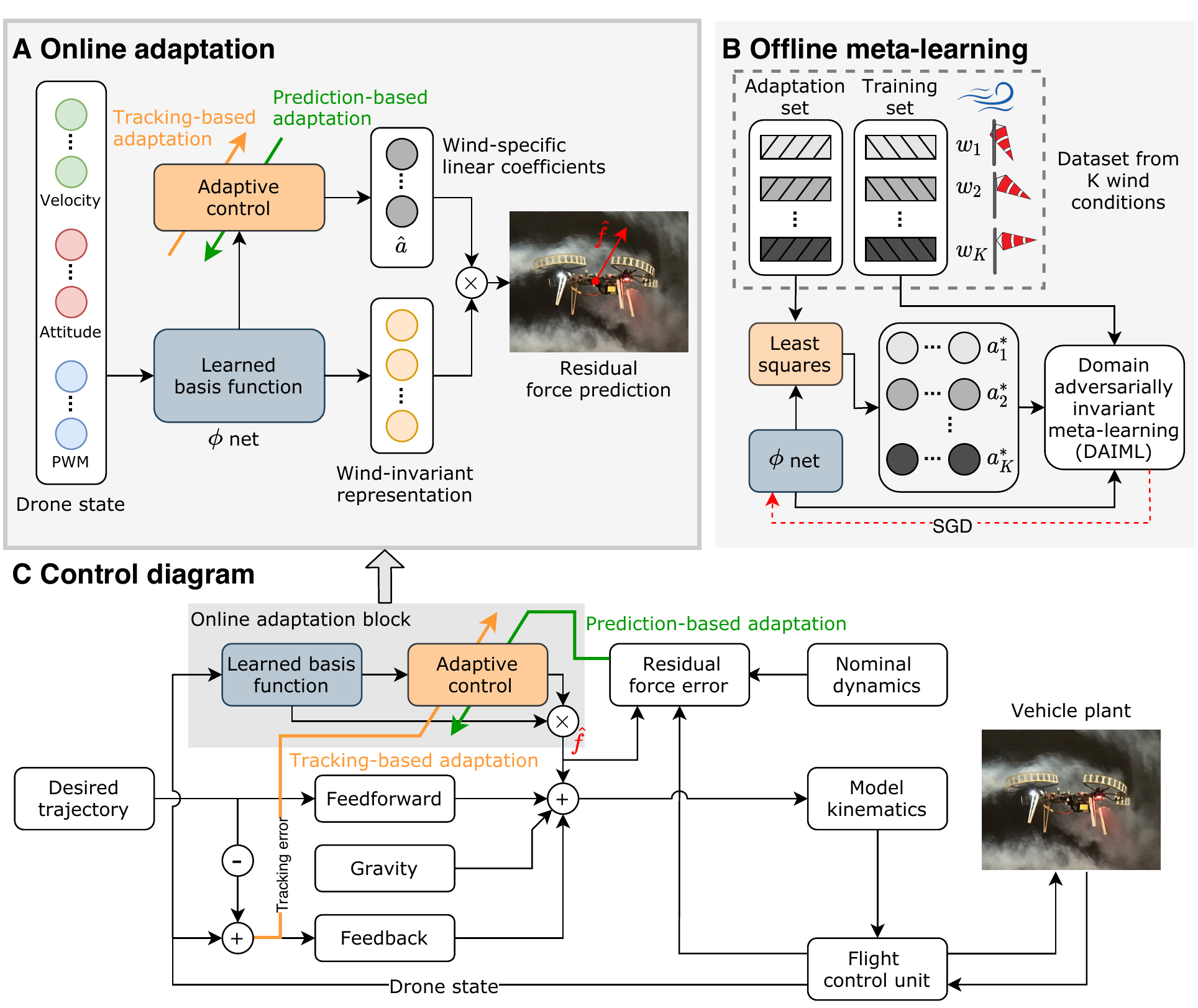}
    \caption{\textbf{Offline meta-learning and online adaptive control design.} (\textbf{A}) The online adaptation block in our adaptive controller. Our controller leverages the meta-trained basis function $\phi$, which is a wind-invariant representation of the aerodynamic effects, and uses composite adaptation (that is, including tracking-error-based and prediction-error-based adaptation) to update wind-specific linear weights $\hat{a}$. The output of this block is the wind-effect force estimate, $\hat{f}=\phi\hat{a}$. (\textbf{B}) The illustration of our meta-learning algorithm \DAML. We collected data from wind conditions $\{w_1,\cdots,w_K\}$ and applied Algorithm \ref{alg:DIML} to train the $\phi$ net. (\textbf{C}) The diagram of our control method, where the grey part corresponds to (A). Interpreting the learned block as an aerodynamic force allows it to be incorporated into the feedback control easily.}
    \label{fig:block-diagram}
\end{figure}

\rev{Using \nf, we report an average improvement of \SI{66}{\%} over a nonlinear tracking controller, \SI{42}{\%} over an $\LOne$ adaptive controller, and \SI{35}{\%} over an Incremental Nonlinear Dynamics Inversion (INDI) controller.}
These results are all accomplished using standard quadrotor UAV hardware, while running the PX4's default regulation attitude control. Our tracking performance is competitive even compared to related work without external wind disturbances and with more complex hardware (for example, \cite{tal_accurate_2021} requires a 10-time higher control frequency and onboard optical sensors for direct motor speed feedback). 
\rev{We also compare \nf~with two variants of our method: \nftransfer, which uses a learned representation trained on data from a different drone, and \nfconstant, which only uses our adaptive control law with a trivial non-learning basis.
\nftransfer~demonstrates that our method is robust to changes in vehicle configuration and model mismatch. \nfconstant, $\LOne$, and INDI all directly adapt to the unknown dynamics without assuming the structure of the underlying physics, and they have similar performance.}
Furthermore, we demonstrate that our method enables a new set of capabilities that allow the UAV to fly through low-clearance gates following agile trajectories in gusty wind conditions (\cref{fig:intro}).


\rev{\subsection*{Related Work for Precise Quadrotor Control}}
\rev{
Typical quadrotor control consists of a cascaded or hierarchical control structure which separates the design of the position controller, attitude controller, and thrust mixer (allocation). Commonly-used off-the-shelf controllers, such as PX4, design each of these loops as proportional-integral-derivative (PID) regulation controllers \cite{meier_pixhawk_2012}. The control performance can be substantially improved by designing each layer of the cascaded controller as a tracking controller using the concept of differential flatness \cite{mellinger_minimum_2011}, or, as has recently been popular, using a single optimization based controller such as model predictive control (MPC) to directly compute motor speed commands from desired trajectories. State-of-the-art tracking performance relies on MPC with fast adaptive inner loops to correct for modeling errors \cite{tal_accurate_2021,hanover_performance_2021}, however, this approach requires full custom flight controllers. In contrast, our method is designed to be integrated with a typical PX4 flight controller, yet it achieves state-of-the-art flight performance in wind.
}

\rev{
Prior work on agile quadrotor control has achieved impressive results by considering aerodynamics~\cite{tal_accurate_2021,hanover_performance_2021,torrente2021data,faessler_differential_2018}. However, those approaches require specialized onboard hardware~\cite{tal_accurate_2021}, full custom flight control stacks~\cite{tal_accurate_2021,hanover_performance_2021}, or cannot adapt to external wind disturbances~\cite{torrente2021data,faessler_differential_2018}.
For example, state-of-the-art tracking performance has been demonstrated using incremental nonlinear dynamics inversion to estimate aerodynamic disturbance forces, with a root-mean-square tracking error of \SI{6.6}{cm} and drone ground speeds up to \SI{12.9}{m/s} \cite{tal_accurate_2021}. However, \cite{tal_accurate_2021} relies on high-frequency control updates (\SI{500}{Hz}) and direct motor speed feedback using optical encoders to rapidly estimate external disturbances. Both are challenging to deploy on standard systems. 
\cite{hanover_performance_2021} simplifies the hardware setup and does not require optical motor speed sensors and has demonstrated state-of-the-art tracking performance. However, \cite{hanover_performance_2021} relies on a high-rate $\LOne$ adaptive controller inside a model predictive controller and uses a racing drone with a fully customized control stack.
\cite{torrente2021data} leverages an aerodynamic model learned offline and represented as Gaussian Processes. 
However, \cite{torrente2021data} cannot adapt to unknown or changing wind conditions and provides no theoretical guarantees. Another recent work focuses on deriving simplified rotor-drag models that are differentially flat \cite{faessler_differential_2018}. 
However, \cite{faessler_differential_2018} focuses on horizontal, $xy-$plane trajectories at ground speeds of \SI{4}{m/s} without external wind, where the thrust is more constant than ours, achieves $\sim$\SI{6}{cm} tracking error \cite{faessler_differential_2018}, uses an attitude controller running at \SI{4000}{Hz}, and is not extensible to faster flights as pointed out in \cite{torrente2021data}.
}\\

\subsection*{Relation between \nf~and Conventional Adaptive Control}
Adaptive control theory has been extensively studied for online control and identification problems with parametric uncertainty, for example, unknown linear coefficients for mixing known basis functions \cite{slotine1991applied,ioannou1996robust,krstic1995nonlinear,narendra2012stable,farrell_adaptive_2006,wise2006adaptive}. There are three common aspects of adaptive control which must be addressed carefully in any well-designed system and which we address in \nf: designing suitable basis functions for online adaptation, stability of the closed-loop system, and persistence of excitation, which is a property related to robustness against disturbances. \revB{These challenges arise due to the coupling between the unknown underlying dynamics and the online adaptation.
This coupling precludes naive combinations of online learning and control. For example, gradient-based parameter adaptation has well-known stability and robustness issues as discussed in \cite{slotine1991applied}. The idea of composite adaptation was first introduced in~\cite{slotine_composite_1989}.
}

The basis functions play a crucial role in the performance of adaptive control, but designing or selecting proper basis functions might be challenging. A good set of basis functions should reflect important features of the underlying physics. In practice, basis functions are often designed using physics-informed modeling of the system, such as the nonlinear aerodynamic modeling in \cite{shi_adaptive_2020}. However, physics-informed modeling requires a tremendous amount of prior knowledge and human labor, and is often still inaccurate. Another approach is to use random features as the basis set, such as random Fourier features \rev{ \cite{rahimi2007random,lale_model_2021}}, which can model all possible underlying physics as long as the number of features is large enough. However, the high-dimensional feature space is not optimal for a specific system because many of the features might be redundant or irrelevant. Such suboptimality and redundancy not only increase the computational burden but also slow down the convergence speed of the adaptation process.

Given a set of basis functions, naive adaptive control designs may cause instability and fragility in the closed-loop system, due to the nontrivial coupling between the adapted model and the system dynamics. 
In particular, asymptotically stable adaptive control cannot guarantee robustness against disturbances and so exponential stability is desired.  Even so, often, existing adaptive control methods only guarantee exponential stability when the desired trajectory is persistently exciting, by which information about all of the coefficients (including irrelevant ones) is constantly provided at the required spatial and time scales. In practice, persistent excitation requires either a succinct set of basis functions or perturbing the desired trajectory, which compromises tracking performance.

\rev{
Recent multirotor flight control methods, including INDI \cite{tal_accurate_2021} and $\LOne$ adaptive control, presented in \cite{mallikarjunan_l1_2012} and demonstrated inside a model predictive control loop in \cite{hanover_performance_2021}, achieve good results by abandoning complex basis functions. Instead, these methods directly estimate the aerodynamic residual force vector. The residual force is observable, thus, these methods bypass the challenge of designing good basis functions and the associated stability and persistent excitation issues. However, these methods suffer from lag in estimating the residual force and encounter the the filter design performance trade of reduced lag versus amplified noise. \nfconstant~only uses \nf's composite adaptation law to estimate the residual force, and therefore, \nfconstant~also falls into this class of adaptive control structures. The results of this \revB{article} demonstrate that the inherent estimation lag in these existing methods limits performance on agile trajectories and in strong wind conditions.
}

\nf~solves the aforementioned issues of basis function design and adaptive control stability, using newly developed methods for meta-learning and composite adaptation that can be seamlessly integrated together. \nf~uses \DAML~and flight data to learn an effective and compact set of basis functions, represented as a DNN. The regularized composite adaptation law uses the learned basis functions to quickly respond to wind conditions. \nf~enjoys fast adaptation because of the conciseness of the feature space, and it guarantees closed-loop exponential stability and robustness without assuming persistent excitation.

Related to \nf, neural network based adaptive control has been researched extensively, but by and large was limited to shallow or single-layer neural networks without pretraining. 
Some early works focus on shallow or single-layer neural networks with unknown parameters which are adapted online \cite{farrell_adaptive_2006, nakanishi_locally_2002,chen1995adaptive,johnson2003limited,narendra1997adaptive}.
A recent work applies this idea to perform an impressive quadrotor flip \cite{bisheban_geometric_2021}. However, the existing neural network based adaptive control work does not employ multi-layer DNNs, and lacks a principled and efficient mechanism to pretrain the neural network before deployment.
Instead of using shallow neural networks, recent trends in machine learning highly rely on DNNs due to their representation power \cite{lecun2015deep}. 
In this work, we leverage modern deep learning advances to pretrain a DNN which represents the underlying physics compactly and effectively.

\subsection*{Related Work in Multi-environment Deep Learning for Robot Control}
Recently, researchers have been addressing the data and computation requirements for DNNs to help the field progress towards the fast online-learning paradigm. In turn, this progress has been enabling adaptable DNN-based control in dynamic environments.  
The most popular learning scheme in dynamic environments is meta-learning, or ``learning-to-learn'', which aims to learn an efficient model from data across different tasks or environments \cite{finn2017model,meta-learning-survey,shi2021meta}. The learned model, typically represented as a DNN, ideally should be capable of rapid adaptation to a new task or an unseen environment given limited data. For robotic applications, meta-learning has shown great potential for enabling autonomy in highly-dynamic environments.
For example, it has enabled quick adaptation against unseen terrain or slopes for legged robots \cite{nagabandi2018learning,song2020rapidly}, changing suspended payload for drones \cite{belkhale2021model}, and unknown operating conditions for wheeled robots \cite{mckinnon2021meta}.

In general, learning algorithms typically can be decomposed into two phases: offline learning and online adaptation. In the offline learning phase, the goal is to learn a model from data collected in different environments, such that the model contains shared knowledge or features across all environment, for example, learning aerodynamic features shared by all wind conditions. In the online adaptation phase, the goal is to adapt the offline-learned model, given limited online data from a new environment or a new task, for example, fine tuning the aerodynamic features in a specific wind condition.

There are two ways that the offline-learned model can be adapted. In the first class, the adaptation phase adapts the whole neural network model, typically using one or more gradient descent steps \cite{finn2017model,nagabandi2018learning,belkhale2021model,clavera2018model}. However, due to the notoriously data-hungry and high-dimensional nature of neural networks, for real-world robots it is still impossible to run such adaptation on-board as fast as the feedback control loop (e.g., $\sim$100Hz for quadrotor). Furthermore, adapting the whole neural network often lacks explainability and robustness and could generate unpredictable outputs that make the closed-loop unstable.

In the second class (including \nf), the online adaptation only adapts a relatively small part of the learned model, for example, the last layer of the neural network \cite{o2021meta,mckinnon2021meta,richards2021adaptive,peng2021linear}.
The intuition is that, different environments share a common representation (e.g., the wind-invariant representation in \cref{fig:block-diagram}(A)), and the environment-specific part is in a low-dimensional space (e.g., the wind-specific linear weight in \cref{fig:block-diagram}(A)), which enables the real-time adaptation as fast as the control loop. 
In particular, the idea of integrating meta-learning with adaptive control is first presented in our prior work \cite{o2021meta}, later followed by \cite{richards2021adaptive}. However, the representation learned in \cite{o2021meta} is ineffective and the tracking performance in \cite{o2021meta} is similar as the baselines; \cite{richards2021adaptive} focuses on a planar and fully-actuated rotorcraft simulation without experiment validation and there is no stability or robustness analysis. \nf~instead learns an effective representation using a \revB{our} meta-learning algorithm called \DAML, demonstrates state-of-the-art tracking performance on real drones, and achieves non-trivial stability and robustness guarantees.

Another popular deep-learning approach for control in dynamic environments is robust policy learning via domain randomization \cite{lee2020learning,tobin2017domain,ramos2019bayessim}. The key idea is to train the policy with random physical parameters such that the controller is robust to a range of conditions. For example, the quadrupedal locomotion controller in~\cite{lee2020learning} retains its robustness over challenging natural terrains. 
However, robust policy learning optimizes average performance under a broad range of conditions rather than achieving precise control by adapting to specific environments.

\section{Results}
\label{sec:results}
In this section, we first discuss the experimental platform for data collection and experiments. Second, we discuss the key conceptual reasoning behind our combined method of our meta-learning algorithm, called \DAML,~and our composite adaptive controller with stability guarantees. Third, we discuss several experiments to quantitatively compare the closed-loop trajectory-tracking performance of our methods to \revision{a nonlinear baseline method and two state-of-the-art adaptive flight control methods, and we observe our methods reduce the average tracking error substantially.} In order to demonstrate the new capabilities brought by our methods, we present agile flight results in gusty winds, where the UAV must quickly fly through narrow gates that are only slightly wider than the vehicle. \revision{Finally, we show our methods are also applicable in outdoor agile tracking tasks without external motion capture systems.}

\begin{figure}[htbp]
    \centering
    \includegraphics[width=0.8\linewidth]{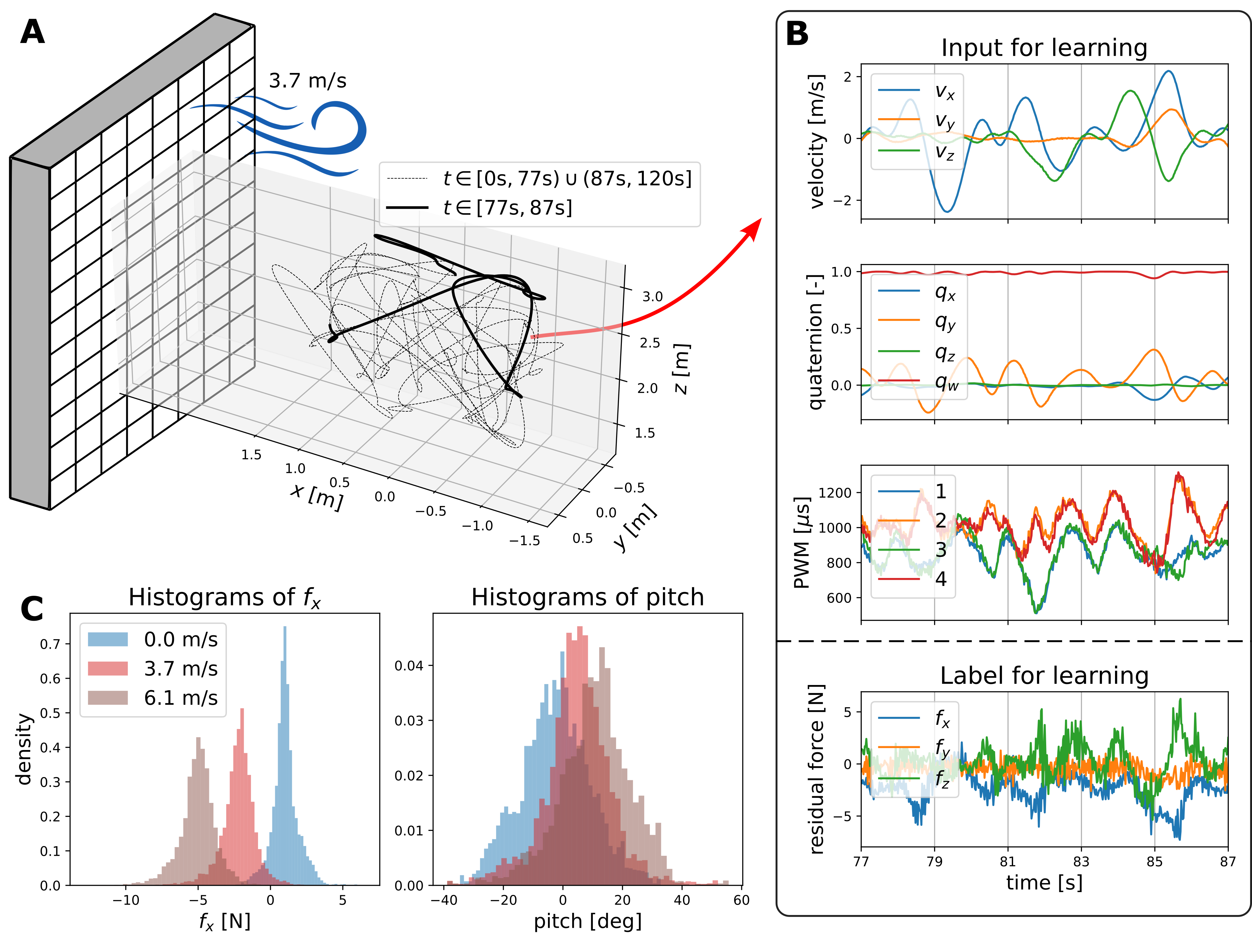}
    \caption{\textbf{Training data collection.} (\textbf{A}) The xyz position along a two-minute randomized trajectory for data collection with wind speed \SI{8.3}{km/h} (\SI{3.7}{m/s}), in the Caltech Real Weather Wind Tunnel. (\textbf{B}) A typical 10-second trajectory of the inputs (velocity, attitude quaternion, and motor speed PWM command) and label (offline calculation of aerodynamic residual force) for our learning model, corresponding to the highlighted part in (A). (\textbf{C}) Histograms showing data distributions in different wind conditions. (\textbf{C}) \textbf{Left:} distributions of the $x$-component of the wind-effect force, $f_x$. This shows that the aerodynamic effect changes as the wind varies. (\textbf{C}) \textbf{Right:} distributions of the pitch, a component of the state used as an input to the learning model. This shows that the shift in wind conditions causes a distribution shift in the input.}
    \label{fig:training-data}
\end{figure}

\subsection*{Experimental Platform}
All of our experiments are conducted at Caltech's Center for Autonomous Systems and Technologies (CAST). The experimental setup consists of an OptiTrack motion capture system with \revision{12 infrared} cameras for localization \rev{streaming position measurements at \SI{50}{Hz}}, a WiFi router for communication, the Caltech Real Weather Wind Tunnel
for generating dynamic wind conditions, and a custom-built quadrotor UAV. 
The Real Weather Wind Tunnel is composed of 1296 individually controlled fans and can generate uniform wind speeds of up to \revision{\highwind} in its 3x3x5\SI{}{m} test section. \revision{For outdoor flight, the drone is also equipped with a Global Positioning System (GPS) module and an external antenna.}
We now discuss the design of the UAV and the wind condition in detail.

\paragraph{UAV Design} We built a quadrotor UAV for our primary data collection and all experiments, shown in \cref{fig:intro}(A).
The quadrotor weighs \revision{\SI{2.6}{kg}} with a thrust to weight ratio of $2.2$. The UAV is equipped with a Pixhawk flight controller running PX4, an open-source commonly used drone autopilot platform \cite{meier_pixhawk_2012}. The UAV incorporates a Raspberry Pi 4 onboard computer running a Linux operation system, which performs real-time computation and adaptive control and interfaces with the flight controller through MAVROS, an open-source set of communication drivers for UAVs. State estimation is performed using the built-in PX4 Extended Kalman Filter (EKF), which fuses inertial measurement unit (IMU) data with global position estimates from OptiTrack motion capture system \revision{(or the GPS module for outdoor flight tasks)}. The UAV platform features a wide-X configuration, measuring \SI{85}{cm} in width, \SI{75}{cm} in length, and \SI{93}{cm} diagonally, and tilted motors for improved yaw authority. \rev{This general hardware setup is standard and similar to many quadrotors.} We refer to the supplementary materials (\ref{sec:supp-drone-configuration}) for further configuration details.

We implemented our control algorithm and the baseline control methods in the position control loop in Python, and run it on the onboard Linux computer at \SI{50}{Hz}. The PX4 was set to the offboard flight mode and received thrust and attitude commands from the position control loop. The built-in PX4 multicopter attitude controller was then executed at the default rate, which is a linear PID regulation controller on the quaternion error. The online inference of the learned representation is also in Python via PyTorch, which is an open source deep learning framework.

To study the generalizability and robustness of our approach, we also use an Intel Aero Ready to Fly drone for data collection. This dataset is used to train a representation of the wind effects on the Intel Aero drone, which we test on our custom UAV. The Intel Aero drone (weighing \SI{1.4}{kg}) has a symmetric X configuration, \SI{52}{cm} in width and \SI{52}{cm} in length, without tilted motors (see the supplementary materials for further details). 

\paragraph{Wind Condition Design} To generate dynamic and diverse wind conditions for the data collection and experiments, we leverage the state-of-the-art Caltech Real Weather Wind Tunnel system (\cref{fig:intro}(A)). The wind tunnel is a \SI{3}{m} by \SI{3}{m} array of $1296$ independently controllable fans capable of generating wind conditions up to \highwind. The distributed fans are controlled in real-time by a Python-based Application Programming Interface (API). For data collection and flight experiments, we designed two types of wind conditions. For the first type, each fan has uniform and constant wind speed between \nowind~and \rev{\highwind~(\highwindalt)}.
The second type of wind follows a sinusoidal function in time, e.g., \rev{\sinwind}. \rev{Note that the training data only covers constant wind speeds up to \SI{6.1}{m/s}.} To visualize the wind, we use $5$ smoke generators to indicate the direction and intensity of the wind condition (see examples in \cref{fig:intro} and Video 1).

\begin{figure}[htbp]
    \centering
    \includegraphics[width=1.0\linewidth]{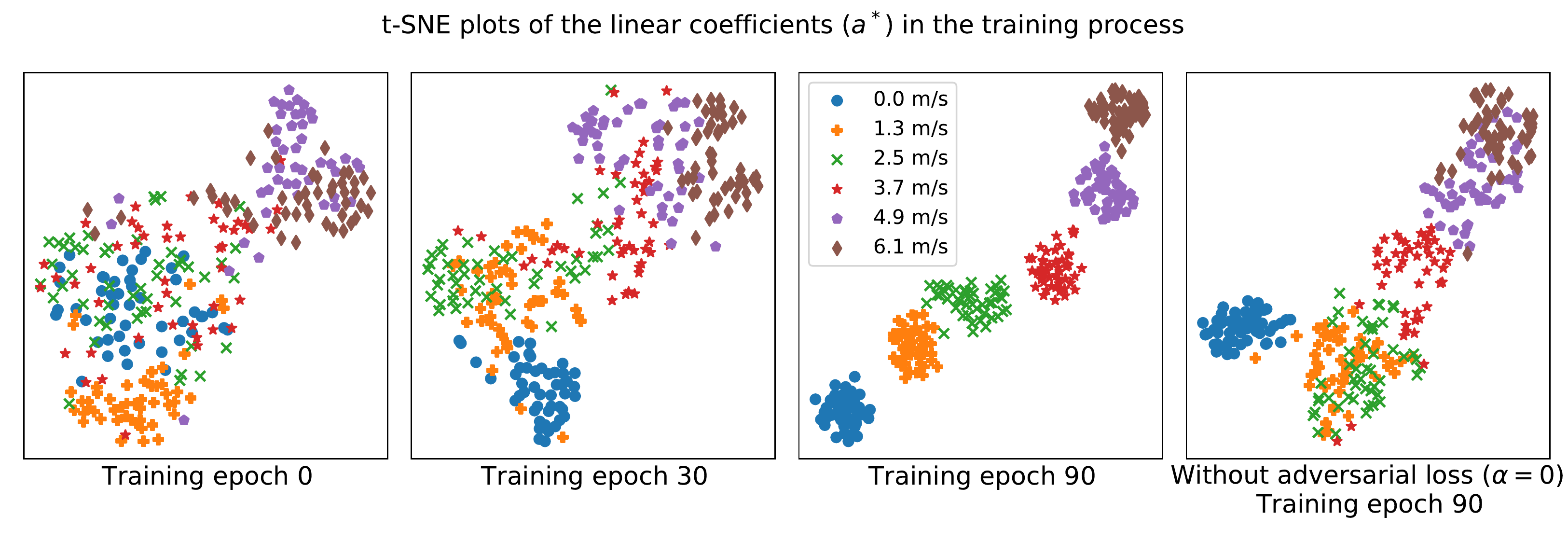}
    \caption{\textbf{t-SNE plots showing the evolution of the linear weights ($a^*$) during the training process.} As the number of training epochs increases, the distribution of $a^*$ becomes more clustered with similar wind speed clusters near each other. The clustering also has a physical meaning: after training convergence, the right top part corresponds to a higher wind speed. This suggests that \DAML~successfully learned a basis function $\phi$ shared by all wind conditions, and the wind-dependent information is contained in the linear weights. Compared to the case without the adversarial regularization term (using $\alpha=0$ in Algorithm \ref{alg:DIML}), the learned result using our algorithm is also more explainable, in the sense that the linear coefficients in different conditions are more disentangled.}
    \label{fig:training-tsne}
\end{figure}

\subsection*{Offline Learning and Online Adaptive Control Development}
\paragraph{Data Collection and Meta-Learning using \DAML} To learn an effective representation of the aerodynamic effects, we have a custom-built drone \rev{follow} a randomized trajectory for 2 minutes each in six different static wind conditions, with speeds ranging from \nowind~to \SI{22.0}{km/h}. \rev{However, in experiments we used wind speeds up to \highwind~(\highwindalt) to study how our methods extrapolate to unseen wind conditions (e.g., \cref{fig:error-wind-plot}).} The data is collected at \SI{50}{Hz} with a total of $36,000$ data points. \Cref{fig:training-data}(A) shows the data collection process, and Fig.~\ref{fig:training-data}(B) shows the inputs and labels of the training data, under one wind condition of \SI{13.3}{km/h}~(\SI{3.7}{m/s}). \Cref{fig:training-data}(C) shows the distributions of input data (pitch) and label data ($x-$component of the aerodynamic force) in different wind conditions. Clearly, a shift in wind conditions causes distribution shifts in both input domain and label domain, which motivates the algorithm design of \DAML. The same data collection process is repeated on the Intel Aero drone, \rev{to study whether the learned representation can generalize to a different drone.}

On the collected datasets for both our custom drone and the Intel Aero drone, we apply the \DAML~algorithm to learn two representations $\phi$ of the wind effects. The learning process is done offline on a normal desktop computer, and depicted in \cref{fig:block-diagram}(B).
\Cref{fig:training-tsne} shows the evolution of the linear coefficients ($a^*$) during the learning process, where \DAML~learns a representation of the aerodynamic effects shared by all wind conditions, and the linear coefficient contains the wind-specific information. Moreover, the learned representation is explainable in the sense that the linear coefficients in different wind conditions are well disentangled (see \cref{fig:training-tsne}).
We refer to the ``\nameref{sec:methods}'' section for more details.

\paragraph{Baselines and the Variants of Our Method} 
We briefly introduce three variants of our method and the \revision{three} baseline methods considered (details are provided in the ``\nameref{sec:methods}'' section). 
\revision{Each of the controllers is implemented in the position control loop and outputs a force command. The force command is fed into a kinematics block to determine a corresponding attitude and thrust, similar to \cite{mellinger_minimum_2011}, which is sent to the PX4 flight controller.}
\revision{The three baselines include: 
globally exponentially-stabilizing nonlinear tracking controller for quadrotor control
\cite{morgan_swarm_2015,shi2019neural,shi2018nonlinear},
incremental nonlinear dynamics inversion (INDI) linear acceleration control \cite{tal_accurate_2021}, and 
$\LOne$ adaptive control \cite{mallikarjunan_l1_2012,hanover_performance_2021}.
The primary difference between these baseline methods and \nf~is how the controller compensates for the unmodelled residual force (that is, each baseline method has the same control structure, in \cref{fig:block-diagram}(C), except for the estimation of the $\hat{f}$). In the case of the nonlinear baseline controller an integral term accumulates error to correct for the modeling error. The integral gain is limited by the stability \rev{of the interaction with the position and velocity error feedback} leading to slow model correction. In contrast, both INDI and $\LOne$ decouple the adaptation rate from the PD gains, which allow for fast adaptation. Instead, these methods are limited by more fundamental design factors, such as system delay, measurement noise, and controller rate.
}

Our method is illustrated in \cref{fig:block-diagram}(A,C) and replaces the integral feedback term with an adapted learning term. 
The deployment of our approach depends on the learned representation function $\phi$, and our primary method and two variants consider a different choice of $\phi$. \nf~is our primary method using a representation learned from the dataset collected by the custom-built drone, which is the same drone used in experiments. \revB{\nftransfer~uses the \nf~algorithm where the representation is trained} using the dataset collected by the aforementioned Intel Aero drone. \revB{\nfconstant~uses the online adaptation algorithm from \nf, but} the representation is an artificially designed constant mapping. \nftransfer~is included to show the generalizability and robustness of our approach with drone transfer, i.e., using a different drone in experiments than data collection. Finally, \nfconstant~demonstrates the benefit of using a better representation learned from the proposed meta-learning method \DAML. Note that \nfconstant~is a composite adaptation form of a Kalman-filter disturbance observer, that is a Kalman-filter \revision{augmented with} a tracking error update term.

\subsection*{Trajectory Tracking Performance}
\revision{We quantitatively compare the performance of the aforementioned control methods when the UAV follows a \SI{2.5}{m} wide, \SI{1.5}{m} tall figure-8 trajectory with a lap time of \SI{6.28}{s} under constant, uniform wind speeds of \nowind, \lowwind~(\lowwindalt), \medwind~(\medwindalt), and \highwind~(\highwindalt) and under time-varying wind speeds of \sinwind~(\sinwindalt).}

\rev{
The flight trajectory for each of the experiments is shown in \cref{fig:trajectory}, which includes a warm up lap and six \SI{6.28}{s} laps. The nonlinear baseline integral term compensates for the mean model error within the first lap. As the wind speed increases, the aerodynamic force variation becomes larger and we notice a substantial performance degradation. INDI and $\LOne$ both improve over the nonlinear baseline, but INDI is more robust than $\LOne$ at high wind speeds. \nfconstant~outperforms INDI except during the two most challenging tasks: \highwind~and sinusoidal wind speeds. The learning based methods, \nf~and \nftransfer, outperform all other methods in all tests. \nf~outperforms \nftransfer~slightly, which is because the learned model was trained on data from the same drone and thus better matches the dynamics of the vehicle.
}

\rev{
In \cref{tab:flight-performance}, we tabulate the root-mean-square position error and mean position error values over the six laps for each experiment.
\Cref{fig:error-wind-plot} shows how the mean tracking error changes for each controller as the wind speed increases, and includes the standard deviation for the mean lap position error. In all cases, \nf~and \nftransfer~outperform the state-of-the-art baseline methods, including the \medwind, \highwind, and sinusoidal wind speeds all of which exceed the wind speed in the training data. 
}
All of these results presents a clear trend: adaptive control substantially outperforms the nonlinear baseline which relies on integral-control, and learning markedly improves adaptive control.

\begin{figure}[htbp]
    \centering
    \includegraphics[width=0.97\linewidth]{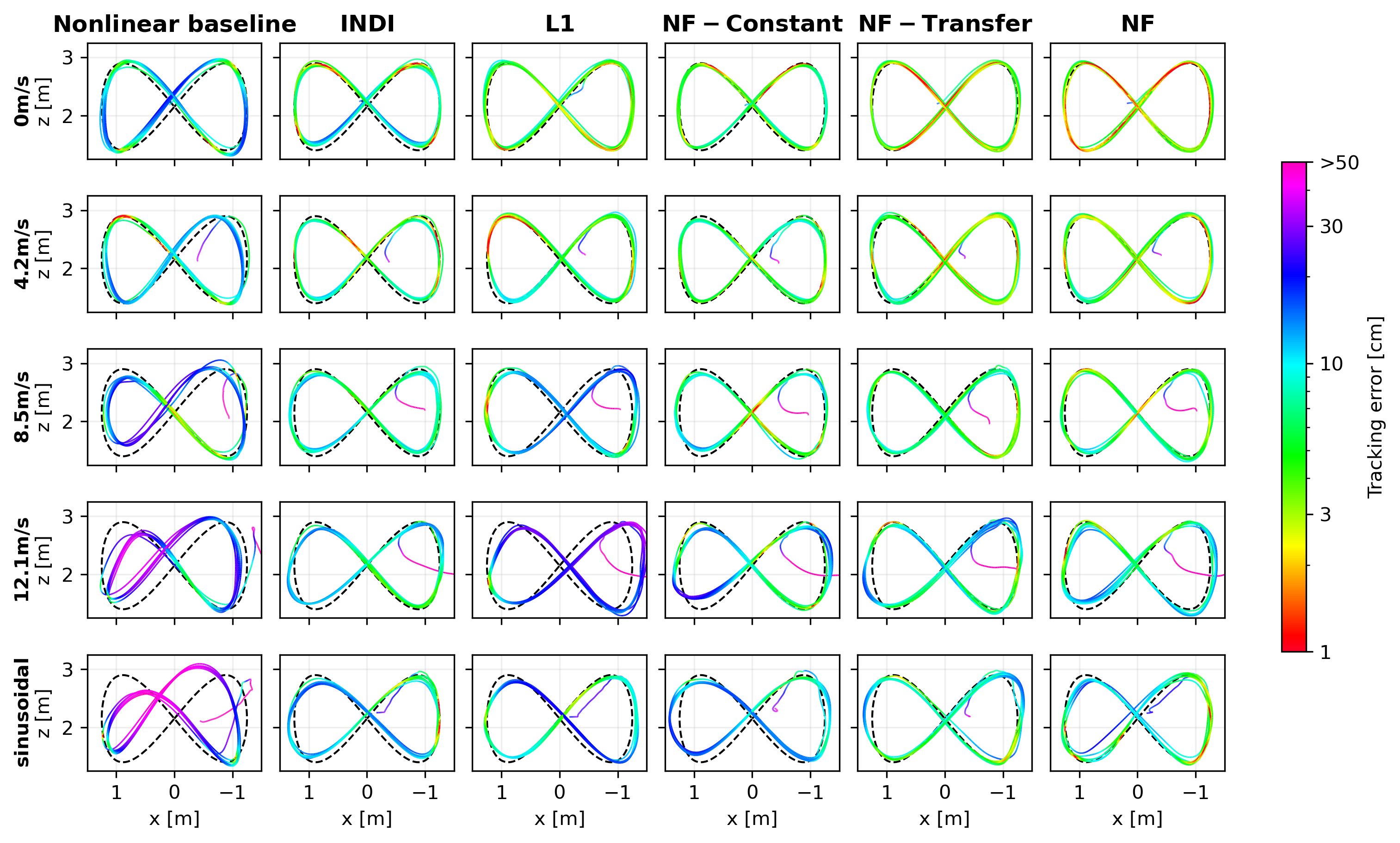}
    \caption{
    \revision{
    \textbf{Depiction of the trajectory tracking performance of each controller in several wind conditions.} The baseline nonlinear controller can track the trajectory well, however, the performance \revB{substantially} degrades at higher wind speeds. INDI, $\LOne$, and \nfconstant~have similar performance and improve over the nonlinear baseline by estimating the aerodynamic disturbance force quickly. \nf~and \nftransfer~use a learned model of the aerodynamic effects and adapt the model in real time to achieve lower tracking error than the other methods.
    }}
    \label{fig:trajectory}
\end{figure}

\subsection*{Agile Flight Through Narrow Gates}
Precise flight control in dynamic and strong wind conditions has many applications, such as rescue and search, delivery, and transportation. In this section, we present a challenging drone flight task in strong winds, where the drone must follow agile trajectories through narrow gates, which are only slightly wider than the drone. The overall result is depicted in \cref{fig:intro} and Video 1. As shown in \cref{fig:intro}(A), the gates used in our experiments are \SI{110}{cm} in width, which is only slightly wider than the drone (\SI{85}{cm} wide, \SI{75}{cm} long). To visualize the trajectory using long-exposure photography, our drone is deployed with four main light emitting diodes (LEDs) on its legs, where the two rear LEDs are red and the front two are white. There are also several small LEDs on the flight controller, the computer, and the motor controllers, which can be seen in the long-exposure shots.

\paragraph{Task Design} 
We tested our method on three different tasks. In the first task (see \cref{fig:intro}(B,D,F-I) and Video 1), the desired trajectory is a \SI{3}{m} by \SI{1.5}{m} figure-8 in the $x-z$ plane with a lap time of \SI{5}{s}. A gate is placed at the left bottom part of the trajectory. The minimum clearance is about \SI{10}{cm} (see \cref{fig:intro}(I), which requires that the controller precisely tracks the trajectory. The maximum speed and acceleration of the desired trajectory are \SI{2.7}{m/s} and \SI{5.0}{m/s^2}, respectively. The wind speed is \SI{3.1}{m/s}. The second task (see Video 1) is the same as the first one, except that it uses a more challenging, time-varying wind condition, $3.1+1.8\sin(\frac{2\pi}{5}t)$\SI{}{m/s}. In the third task (see \cref{fig:intro}(C,E) and Video 1), the desired trajectory is a \SI{3}{m} by \SI{2.5}{m} ellipse in the $x-y$ plane with a lap time of \SI{5}{s}. We placed two gates on the left and right sides of the ellipse. As with the first task, the wind speed is \SI{3.1}{m/s}.

\paragraph{Performance} For all three tasks, we used our primary method, \nf, where the representation is learned using the dataset collected by the custom-built drone. \Cref{fig:intro}(D,E) are two long-exposure photos with an exposure time of \SI{5}{s}, which is the same as the lap time of the desired trajectory. We see that our method precisely tracked the desired trajectories and flew safely through the gates (see Video 1). These long-exposure photos also captured the smoke visualization of the wind condition. We would like to emphasize that the drone is wider than the LED light region, since the LEDs are located on the legs (see \cref{fig:intro}(A)).  \Cref{fig:intro}(F-I) are four high-speed photos with a shutter speed of 1/200\SI{}{s}. These four photos captured the moment the drone passed through the gate in the first task, as well as the complex interaction between the drone and the wind. We see that the aerodynamic effects are complex and non-stationary and depend on the UAV attitude, the relative velocity, and aerodynamic interactions between the propellers and the wind.

\begin{figure} 
    \centering
    \includegraphics[width=0.8\linewidth]{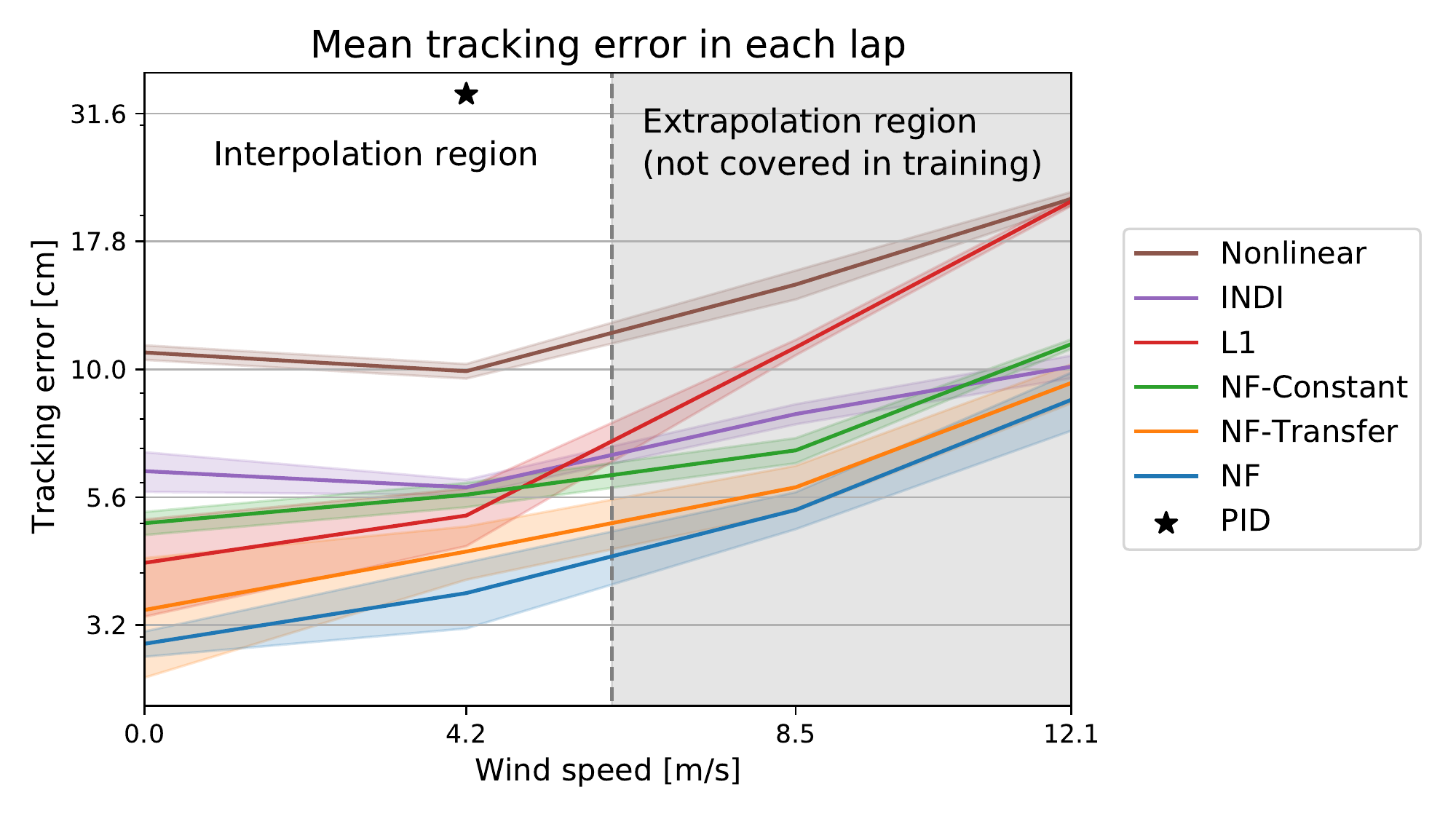}
    \caption{\revision{\textbf{Mean tracking errors of each lap in different wind conditions.} This figure shows position tracking errors of different methods as wind speed increases. Solid lines show the mean error over 6 laps and the shade areas show standard deviation of the mean error on each lap. The grey area indicates the extrapolation region, where the wind speeds are not covered in training. Our primary method (\nf) achieves state-of-the-art performance even with a strong wind disturbance.}}
    \label{fig:error-wind-plot}
\end{figure}

\begin{table}
\caption{\revision{Tracking error statistics in \SI{}{cm} for different wind conditions. Two metrics are considered: root-mean-square (RMS) and mean.}}
\label{tab:flight-performance}
\centering
\small
    \begin{tabular}{|c|cc|cc|cc|cc|cc|}
        \hline
        \multirow{3}{*}{\backslashbox{Method}{Wind speed\\ {[m/s]}}} & \multicolumn{2}{c|}{\multirow{2}{*}{0}}
        & \multicolumn{2}{c|}{\multirow{2}{*}{4.2}}
        & \multicolumn{2}{c|}{\multirow{2}{*}{8.5}}
        & \multicolumn{2}{c|}{\multirow{2}{*}{12.1}}
        & \multicolumn{2}{c|}{\multirow{2}{*}{$8.5+2.4\sin(t)$}} \\
        &&&&&&&&&&\\
        & RMS & Mean
        & RMS & Mean
        & RMS & Mean
        & RMS & Mean
        & RMS & Mean \\
        \hline
        
        \textbf{Nonlinear}
            &  11.9 &  10.8
            &  10.7 &  9.9
            &  16.3 &  14.7
            &  23.9 &  21.6
            &  31.2 &  28.2
        \\ \textbf{INDI}
            &  7.3 &  6.3
            &  6.4 &  5.9
            &  8.5 &  8.2
            &  10.7 &  10.1
            &  11.1 &  10.3
        \\ \textbf{L1}
            &  4.6 &  4.2
            &  5.8 &  5.2
            &  12.1 &  11.1
            &  22.7 &  21.3
            &  13.0 &  11.6
        \\ \textbf{NF-Constant}
            &  5.4 &  5.0
            &  6.1 &  5.7
            &  7.5 &  6.9
            &  12.7 &  11.2
            &  12.7 &  12.1
        \\ \textbf{NF-Transfer}
            &  3.7 &  3.4
            &  4.8 &  4.4
            &  6.2 &  5.9
            &  10.2 &  9.4
            &  8.8 &  8.0
        \\ \textbf{NF}
            &  \textbf{3.2} &  \textbf{2.9}
            &  \textbf{4.0} &  \textbf{3.7}
            &  \textbf{5.8} &  \textbf{5.3}
            &  \textbf{9.4} &  \textbf{8.7}
            &  \textbf{7.6} &  \textbf{6.9}
        \\ \hline
    \end{tabular}
\end{table}

\rev{\subsection*{Outdoor Experiments}
We tested our algorithm outdoors in gentle breeze conditions (wind speeds measured up to \SI{17}{km/h}). An onboard GPS receiver provided position information to the EKF, giving lower precision state estimation, and therefore less precise aerodynamic residual force estimation. Following the same aforementioned figure-8 trajectory, the controller reached \SI{7.5}{cm} mean tracking error, shown in \cref{fig:outdoor}.
}








\section{Discussion}
\label{sec:discussion}
\subsection*{State-of-the-art Tracking Performance}
When measuring position tracking errors, \rev{we observe that our \nf~method outperforms state-of-the-art flight controllers in all wind conditions}. 
\nf~uses deep learning to obtain a compact representation of the aerodynamic disturbances and incorporates \rev{that representation} into an adaptive control design to achieve high precision tracking performance. 
\rev{The benchmark methods used in this \revB{article} are nonlinear control, INDI, and $\LOne$ and performance is compared tracking an agile figure-8 in constant and time-varying wind speeds up to \highwind~(\highwindalt). 
Furthermore, we observe a mean tracking error of \SI{2.9}{cm} in \nowind~wind, which is comparable with state-of-the-art tracking performance demonstrated on more aggressive racing drones \cite{tal_accurate_2021,hanover_performance_2021} despite several architectural limitations such as limited control rate in offboard mode, a larger, less maneuverable vehicle, and without direct motor speed measuremnts. All our experiments were conducted using the standard PX4 attitude controller, with \nf~implemented in an onboard, low cost, and ``credit-card sized" Raspberry Pi 4 computer.}
Furthermore, \nf~is robust to changes in vehicle configuration, as demonstrated by the similar performance of \nftransfer. 

To understand the fundamental tracking-error limit, we estimate that the localization precision from the OptiTrack system is about \SI{1}{cm}, which is a practical lower bound for the average tracking error in our system (see more details in the supplementary material, \ref{sec:supp-localization-error}). This is based on the fact that the difference between the OptiTrack position measurement and the onboard EKF position estimate is around \SI{1}{cm}. 

To achieve a tracking error of \SI{1}{cm}, remaining improvements should focus on reducing code execution time, communication delays, and attitude tracking delay. We measured the combined code execution time and communication delay to be at least \SI{15}{ms} and often as much as \SI{30}{ms}. A faster implementation (such as using C++ instead of Python) and streamlined communication layer (such as using ROS2's real-time features) could allow us to achieve tracking errors on the order of the localization accuracy. Attitude tracking delay can be \revB{substantially} reduced through the use of a nonlinear attitude controller (e.g., \cite{shi2018nonlinear}). \rev{Our method is also directly extensible to attitude control because attitude dynamics match the Euler-Lagrange dynamics used in our derivations.} However, further work is needed to understand the interaction of the learned dynamics with the cascaded control design when implementing a tracking attitude controller. 

We have tested our control method in outdoor flight to demonstrate that it is robust to less precise state estimation and does not rely on any particular features of our test facility.
Although control and estimation are usually separately designed parts of an autonomous system, aggressive adaptive control requires minimal noise in force measurement to effectively and quickly compensate for unmodelled dynamics. Testing our method in outdoor flight, the quadrotor maintains precise tracking with only \SI{7.5}{cm} tracking error on a gentle breezy day with wind speeds around \SI{17}{km/h}.

\begin{figure}
    \centering
    \begin{subfigure}[c]{0.35\linewidth}
    \centering
    \includegraphics[width=\textwidth]{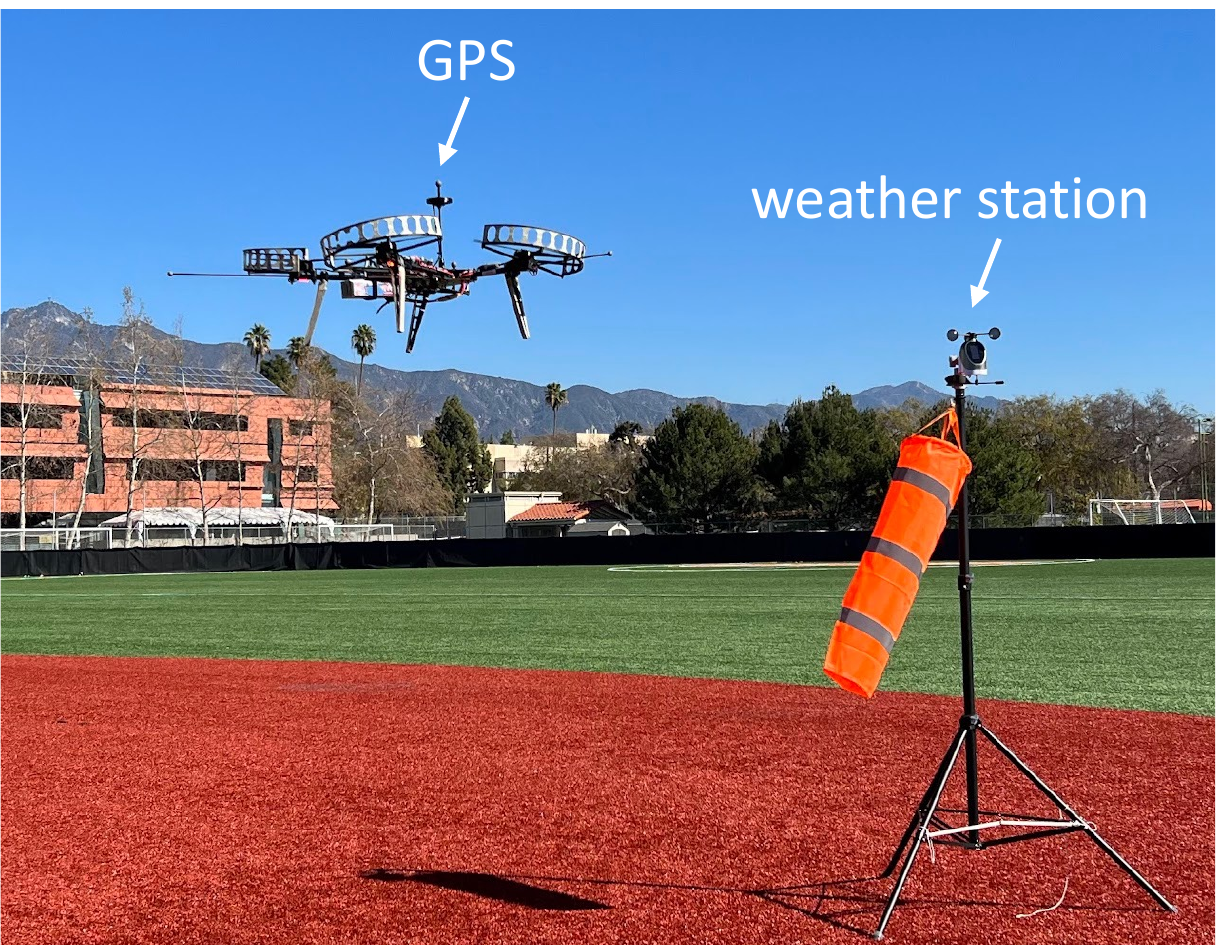}
    \end{subfigure}
    ~
    \begin{subfigure}[c]{0.62\linewidth}
    \centering
    \includegraphics[width=\textwidth]{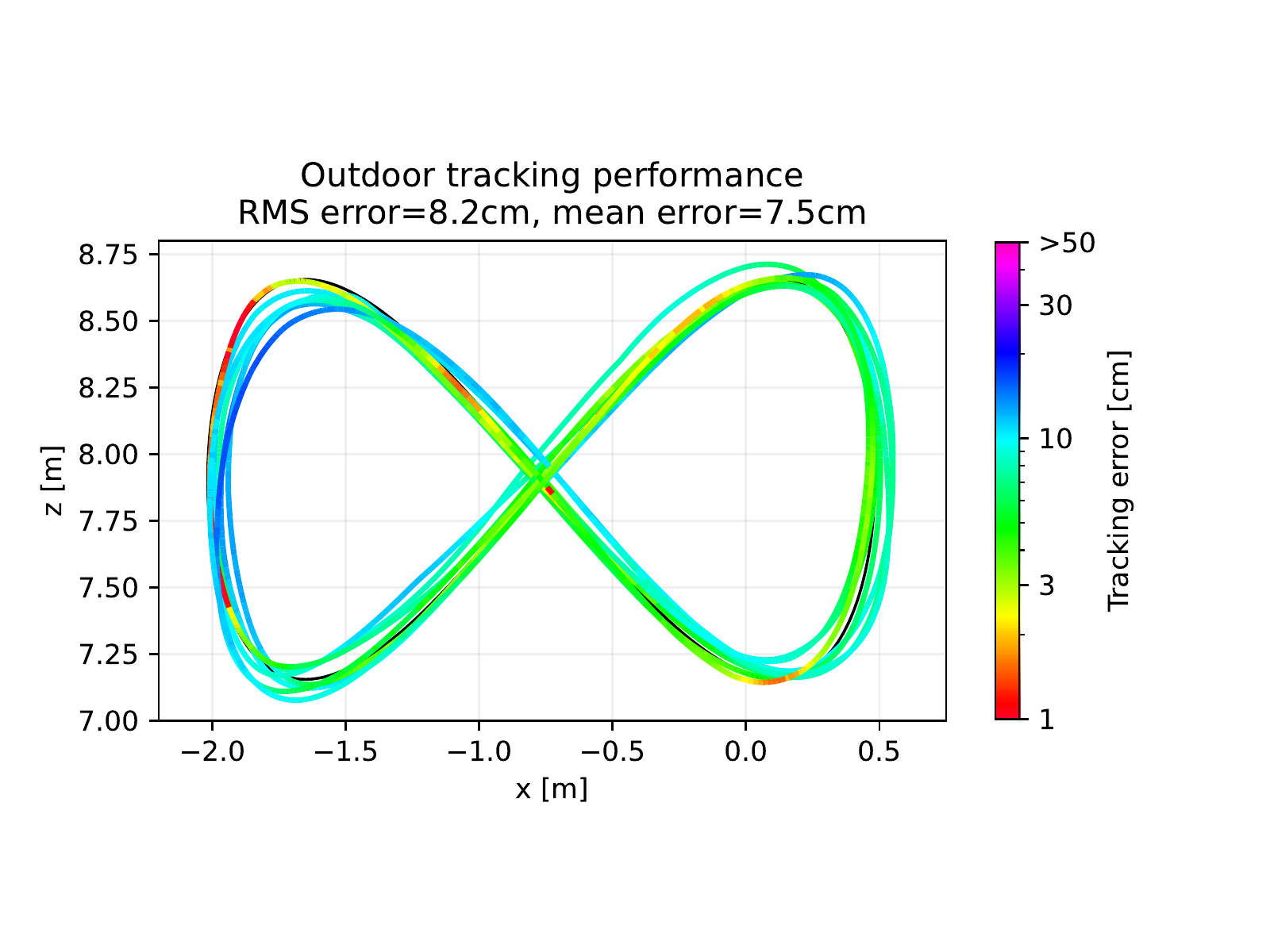}
    \end{subfigure}
    
    \vspace{-1.2cm}
    \caption{\revision{\textbf{Outdoor flight setup and performance.} \textbf{Left:} In outdoor experiments, a GPS module is deployed for state estimation, and a weather station records wind profiles. The maximum wind speed during the test was around \SI{17}{km/h} (\SI{4.9}{m/s}). \textbf{Right:} Trajectory tracking performance of \nf.}}
    \label{fig:outdoor}
\end{figure}

\subsection*{Challenges Caused by Unknown and Time-varying Wind Conditions}
In the real world, the wind conditions are not only unknown but also constantly changing, and the vehicle must continuously adapt. We designed the the sinusoidal wind test to emulate unsteady or gusty wind conditions.
Although our learned model is trained on static and approximately uniform wind condition data, \nf~can quickly identify changing wind speed and maintains precise tracking even \rev{on the sinusoidal wind experiment}.
\rev{Moreover,} in each of our experiments, the wind conditions were unknown to the UAV before starting the test yet were quickly identified by the adaptation algorithm. 

\revB{Our work demonstrated that it is possible to} repeatably and quantitatively test quadrotor flight in time-varying wind. Our method separately learns the wind effect's dependence on the vehicle state (i.e., the wind-invariant representation in \cref{fig:block-diagram}(A)) and the wind condition (i.e., the wind-specific linear weight in \cref{fig:block-diagram}(A)). This separation allows \nf~to quickly adapt to the time-varying wind even as the UAV follows a dynamic trajectory, with an average tracking error below \rev{\SI{8.7}{cm}} in \cref{tab:flight-performance}.  

\subsection*{Computational Efficiency of Our Method}
In the offline meta-learning phase, the proposed \DAML~algorithm is able to learn an effective representation of the aerodynamic effect in a data efficient manner. This requires only 12 minutes of flight data at \SI{50}{Hz}, for a total of 36,000 data points. The training procedure only takes 5 minutes on a normal desktop computer. In the online adaptation phase, our adaptive control method only takes \SI{10}{ms} to compute on a compact onboard Linux computer (Raspberry Pi 4). In particular, the feedforward inference time via the learned basis function is about \SI{3.5}{ms} and the adaptation update is about \SI{3.0}{ms}, which implies the compactness of the learned representation. 

\subsection*{Generalization to New Trajectories and New Aircraft}
It is worth noting that our control method is orthogonal to the design of the desired trajectory. In this \revB{article}, we \rev{focus on the figure-8 trajectory which is a commonly used control benchmark. We also demonstrate our method flying a horizontal ellipse during the narrow gate demonstration \cref{fig:intro}.} Note that our method supports any trajectory planners such as \cite{foehn2021time} or learning-based planners \cite{nakka2020chance,doi:10.1126/scirobotics.abg5810,liu2020robust}. In particular, for those planners which require a precise and agile downstream controller (e.g., for close-proximity flight or drone racing \cite{foehn2021time,shi2021neural}), our method immediately provides a solution and further pushes the boundary of these planners, because our state-of-the-art tracking capabilities enable tighter configurations and smaller clearances. However, further research is required to understand the coupling between planning and learning-based control near actuation limits. \rev{Future work will consider using \nf~in a combined planning and control structure such as MPC, which will be able to handle actuation limits.}

The comparison between \nf~and \nftransfer~show that our approach is robust to changing vehicle design and the learned representation \revB{does not depend on the} vehicle. This demonstrates the generalizability of the proposed method running on different quadrotors. Moreover, our control algorithm is formulated generally for all robotic systems described by the Euler-Langrange equation (see ``\nameref{sec:methods}''), including many types of aircraft such as \cite{shi_adaptive_2020,doi:10.1126/scirobotics.abf8136}. 

%
%

\section{Materials and Methods}
\label{sec:methods} 

\subsection*{Overview}
We consider a general robot dynamics model:
\begin{equation}
    \label{eq:open-loop-dynamics}
    \begin{aligned}
    M(q)\ddot q + C(q,\dot q) \dot q + g(q) &= u + f(q, \dot{q}, w) 
    \end{aligned}
\end{equation}
where $q, \dot q, \ddot q \in \mathbb{R}^n$ are the $n$ dimensional position, velocity, and acceleration vectors, $M(q)$ is the symmetric, positive definite inertia matrix, $C(q,\dot q)$ is the Coriolis matrix, $g(q)$ is the gravitational force vector and $u\in\mathbb{R}^n$ is the control force. Most importantly, $f(q,\dot q, w)$ incorporates unmodeled dynamics, and $w \in \mathbb{R}^m$ is an unknown hidden state used to represent the underlying environmental conditions, which is potentially time-variant. Specifically, in this \revB{article}, $w$ represents the wind profile (for example, the wind profile in \cref{fig:intro}), and different wind profiles yield different unmodeled aerodynamic disturbances for the UAV.

\nf~can be broken into two main stages, the offline meta-learning stage and the online adaptive control stage. These two stages build a model of the unknown dynamics of the form 
\begin{align}
    f(q, \dot q, w) &\approx \phi(q, \dot q) a(w),
\end{align}
where $\phi$ is a basis or representation function shared by all wind conditions and captures the dependence of the unmodeled dynamics on the robot state, and $a$ is a set of linear coefficients that is updated for each condition. In the supplementary material (\ref{sec:supp-learning-expressiveness}), we prove that the decomposition $\phi(q,\dot{q})a(w)$ exists for any analytic function $f(q,\dot{q},w)$.
In the offline meta-learning stage, we learn $\phi$ as a DNN using our meta-learning algorithm \DAML. This stage results in learning $\phi$ as a wind-invariant representation of the unmodeled dynamics, which generalizes to new trajectories and new wind conditions.
In the online adaptive control stage, we adapt the linear coefficients $a$ using adaptive control. Our adaptive control algorithm is a type of composite adaptation and was carefully designed to allow for fast adaptation while maintaining the global exponential stability and robustness of the closed loop system. The offline learning and online control architectures are illustrated in \cref{fig:block-diagram}(B) and \cref{fig:block-diagram}(A,C), respectively.

\subsection*{Data Collection}
To generate training data to learn a wind-invariant representation of the unmodeled dynamics, the drone tracks a randomized trajectory with the baseline nonlinear controller for 2 minutes each in several different static wind conditions. \Cref{fig:training-data}(A) illustrates one trajectory under the wind condition \SI{13.3}{km/h}~(\SI{3.7}{m/s}). The set of input-output pairs for the \kth such trajectory is referred as the \kth subdataset, $D_{w_k}$, with the wind condition $w_k$. Our dataset consists of 6 different subdatasets with wind speeds from \SI{0}{km/h} to \SI{22.0}{km/h}~(\SI{6.1}{m/s}), which are in the white interpolation region in \cref{fig:error-wind-plot}. 

The trajectory follows a polynomial spline between 3 waypoints: the current position and two randomly generated target positions. The spline is constrained to have zero velocity, acceleration, and jerk at the starting and ending waypoints. Once the end of one spline is reached, the a new random spline is generated and the process repeats for the duration of the training data flight. This process allows us to generate a large amount of data using a trajectory very different from the trajectories used to test our method, such as the figure-8 in \cref{fig:intro}. By training and testing on different trajectories, we demonstrate that the learned model generalizes well to new trajectories. 

Along each trajectory, we collect time-stamped data $[q,\dot{q},u]$.  Next, we compute the acceleration $\ddot{q}$ by fifth-order numerical differentiation.  Combining this acceleration with \cref{eq:open-loop-dynamics}, we get a noisy measurement of the unmodeled dynamics, $y = f(x, w) + \epsilon$, where $\epsilon$ includes all sources of noise (e.g., sensor noise and noise from numerical differentiation) and $x = [q; \dot{q}] \in \mathbb{R}^{2n}$ is the state. Finally, this allows us to define the dataset, $\mathcal{D}=\{D_{w_1},\cdots,D_{w_K}\}$, where
\begin{equation}
D_{w_k}=\left\{x_k^{(i)}, y_k^{(i)}=f(x_k^{(i)},w_k)+\epsilon_k^{(i)} \right\}_{i=1}^{N_k}   
\end{equation}
is the collection of $N_k$ noisy input-output pairs with wind condition $w_k$. As we discuss in the ``\nameref{sec:results}'' section, in order to show \DAML~
\rev{learns a model which can be transferred between drones, }
we applied this data collection process on both the custom built drone and the Intel Aero RTF drone.

\subsection*{The \LearningAlg~(\DAML) Algorithm}
In this section, we will present the methodology and details of learning the representation function $\phi$. In particular, we will first introduce the goal of meta-learning, motivate the proposed algorithm \DAML~by the observed domain shift problem from the collected dataset, and finally discuss key algorithmic details.

\paragraph{Meta-Learning Goal} Given the dataset, the goal of meta-learning is to learn a representation $\phi(x)$, such that for any wind condition $w$, there exists a latent variable $a(w)$ which allows $\phi(x)a(w)$ to approximate $f(x,w)$ well. Formally, an optimal representation, $\phi$, solves the following optimization problem:
\begin{equation}
\label{eq:optimization-f-loss}
    \min_{\phi,a_1,\cdots,a_K}\sum_{k=1}^K\sum_{i=1}^{N_k}
    \left\| y_k^{(i)} - \phi(x_k^{(i)})a_k \right\|^2,
\end{equation}
where $\phi(\cdot):\B{R}^{2n}\rightarrow\B{R}^{n\times h}$ is the representation function and $a_k\in\B{R}^h$ is the latent linear coefficient. Note that the optimal weight $a_k$ is specific to each wind condition, but the optimal representation $\phi$ is shared by all wind conditions. In this \revB{article}, we use a deep neural network (DNN) to represent $\phi$. In the supplementary material (Section S2), we prove that for any analytic function $f(x,w)$, the structure $\phi(x)a(w)$ can approximate $f(x,w)$ with an arbitrary precision, as long as the DNN $\phi$ has enough neurons. This result implies that the $\phi$ solved from the optimization in \cref{eq:optimization-f-loss} is a reasonable representation of the unknown dynamics $f(x,w)$. 

\paragraph{Domain Shift Problems} One challenge of the optimization in \cref{eq:optimization-f-loss} is the \emph{inherent domain shift} in $x$ caused by the shift in $w$. Recall that during data collection we have a program flying the drone in different winds. The actual flight trajectories differ vastly from wind to wind because of the wind effect.
Formally, the distribution of $x_k^{(i)}$ varies between $k$ because the underlying environment or context $w$ has changed. For example, as depicted by \cref{fig:training-data}(C), the drone pitches into the wind, and the average degree of pitch depends on the wind condition. Note that pitch is only one component of the state $x$. The domain shift in the whole state $x$ is even more drastic.

Such inherent shifts in $x$ bring challenges for deep learning. The DNN $\phi$ may memorize the distributions of $x$ in different wind conditions, such that the variation in the dynamics $\{ f(x,w_1), f(x,w_2), \cdots, f(x,w_K)\}$ is reflected via the distribution of $x$, rather than the wind condition $\{w_1, w_2, \cdots, w_K\}$. In other words, the optimization in \cref{eq:optimization-f-loss} may lead to \emph{over-fitting} and may not properly find a wind-invariant representation $\phi$.

To solve the domain shift problem, inspired by \cite{ganin2016domain}, we propose the following adversarial optimization framework:
\begin{equation}
\label{eq:optimization-both-loss}
\begin{aligned}
    \max_{h} \min_{\phi,a_1,\cdots,a_K}
    \sum_{k=1}^K\sum_{i=1}^{N_k}
    \left(
    \left\| y_k^{(i)} - \phi(x_k^{(i)})a_k \right\|^2 - \alpha\cdot\mathrm{loss}\left( h( \phi(x_k^{(i)}) ), k \right)
    \right),
\end{aligned}
\end{equation}
where $h$ \rev{is another DNN that works as a discriminator to} predict the environment index out of $K$ wind conditions, $\mathrm{loss}(\cdot)$ is a classification loss function (e.g., the cross entropy loss), $\alpha\geq0$ is a hyperparameter to control the degree of regularization, $k$ is the wind condition index, and $(i)$ is the input-output pair index. Intuitively, $h$ and $\phi$ play a zero-sum max-min game: the goal of $h$ is to predict the index $k$ directly from $\phi(x)$ (achieved by the outer $\max$); the goal of $\phi$ is to approximate the label $y_k^{(i)}$ while making the job of $h$ harder (achieved by the inner $\min$). In other words, $h$ is a learned regularizer to remove the environment information contained in $\phi$. In our experiments, the output of $h$ is a $K-$dimensional vector for the classification probabilities of $K$ conditions, and we use the cross entropy loss for $\mathrm{loss}(\cdot)$, which is given as
\begin{align}
    \mathrm{loss}\left( h( \phi(x_k^{(i)}) ), k \right) = - \sum_{j=1}^K \delta_{k j} \log\left( h( \phi(x_k^{(i)}) )^\top e_j \right)
\end{align}
where $\delta_{kj}=1$ if $k=j$ and $\delta_{kj}=0$ otherwise and $e_j$ is the standard basis function.

\begin{algorithm}[t]
    \caption{\LearningAlg~(\DAML)}
    \label{alg:DIML}
    \DontPrintSemicolon
    \SetKwInput{KwSolver}{Solver}
    \SetKwInput{KwObserve}{Observe}
    \SetKwInput{KwPara}{Hyperparameter}
    \SetKwInput{KwResult}{Result}
    \SetKwInput{KwInitialize}{Initialize}
    \SetKwComment{CommentT}{$\triangleright$\ }{}
    \SetKwRepeat{RepeatConvergence}{repeat}{until convergence}
    \SetKwFunction{Rand}{rand}
    
    \KwPara{$\alpha\geq0$, $0<\eta\leq1$, $\gamma>0$}
    \KwIn{$\mathcal{D}=\{D_{w_1},\cdots,D_{w_K}\}$}
    \KwInitialize{Neural networks $\phi$ and $h$}
    \KwResult{Trained neural networks $\phi$ and $h$}
    
    \RepeatConvergence{}{
        Randomly sample $D_{w_k}$ from $\mathcal{D}$ \\
        Randomly sample two disjoint batches $B^a$ (adaptation set) and $B$ (training set) from $D_{w_k}$ \\
        Solve the least squares problem $a^*(\phi)=\arg\min_a \sum_{i\in B^a} \left\| y_k^{(i)} - \phi(x_k^{(i)}) a \right\|^2$ \label{alg:DIML:LS} \\
        \If{$\|a^*\|>\gamma$}{
            $a^*\leftarrow \gamma\cdot\frac{a^*}{\|a^*\|}\quad$ \label{alg:DIML:normalization} \CommentT{normalization}
        }
        Train DNN $\phi$ using stochastic gradient descent (SGD) and spectral normalization with loss 
        \begin{equation*}
        \begin{aligned}
            \sum_{i\in B} \left( \left\| y_k^{(i)} - \phi(x_k^{(i)})a^*\right\|^2 
            - \alpha\cdot\mathrm{loss}\left( h(\phi(x_k^{(i)})), k \right) \right)
        \end{aligned}
        \end{equation*}\label{alg:DIML:learning_phi}\\
        \If{\Rand{} $\leq \eta$}{
           Train DNN $h$ using SGD with loss $\sum_{i\in B}\mathrm{loss}\left( h(\phi(x_k^{(i)})), k \right)$ \label{alg:DIML:learning_h}
        }
    }
\end{algorithm}

\paragraph{Design of the \DAML~Algorithm} Finally, we solve the optimization problem in \cref{eq:optimization-both-loss} by the proposed algorithm \DAML~(described in Algorithm \ref{alg:DIML} and illustrated in \cref{fig:block-diagram}(B)), which belongs to the category of gradient-based meta-learning \cite{meta-learning-survey}, but with least squares as the adaptation step. \DAML~contains three steps: (i) The adaptation step (Line \ref{alg:DIML:LS}-6) solves an least squares problem as a function of $\phi$ on the adaptation set $B^a$. (ii) The training step (Line \ref{alg:DIML:learning_phi}) updates the learned representation $\phi$ on the training set $B$, based on the optimal linear coefficient $a^*$ solved from the adaptation step. (iii) The regularization step (Line 8-\ref{alg:DIML:learning_h}) updates the discriminator $h$ on the training set.

We emphasize important features of \DAML:
(i) After the adaptation step, $a^*$ is a function of $\phi$. In other words, in the training step (Line \ref{alg:DIML:learning_phi}), the gradient with respect to the parameters in the neural network $\phi$ will backpropagate through $a^*$. Note that the least-square problem (Line \ref{alg:DIML:LS}) can be solved efficiently with a closed-form solution.
(ii) The normalization (Line \ref{alg:DIML:normalization}) is to make sure $\|a^*\|\leq\gamma$, which improves the robustness of our adaptive control design. We also use spectral normalization in training $\phi$, to control the Lipschitz property of the neural network and improve generalizability~\cite{shi2019neural,shi2021neural,bartlett2017spectrally}. 
(iii) We train $h$ and $\phi$ in an alternating manner. In each iteration, we first update $\phi$ (Line \ref{alg:DIML:learning_phi}) while fixing $h$ and then update $h$ (Line \ref{alg:DIML:learning_h}) while fixing $\phi$. However, the probability to update the discriminator $h$ in each iteration is $\eta\leq1$ instead of $1$, to improve the convergence of the algorithm \cite{goodfellow2014generative}. 

We further motivate the algorithm design using \cref{fig:training-data} and \cref{fig:training-tsne}. \Cref{fig:training-data}(A,B) shows the input and label from one wind condition, and \cref{fig:training-data}(C) shows the distributions of the pitch component in input and the $x-$component in label, in different wind conditions. The distribution shift in label implies the importance of meta-learning and adaptive control, because the aerodynamic effect changes drastically as the wind condition switches. On the other hand, the distribution shift in input motivates the need of \DAML. \Cref{fig:training-tsne} depicts the evolution of the optimal linear coefficient ($a^*$) solved from the adaptation step in \DAML, via the t-distributed stochastic neighbor embedding (t-SNE) dimension reduction, which projects the 12-dimensional vector $a^*$ into 2-d. The distribution of $a^*$ is more and more clustered as the number of training epochs increases. In addition, the clustering behavior in \cref{fig:training-tsne} has a concrete physical meaning: right top part of the t-SNE plot corresponds to a higher wind speed. These properties imply the learned representation $\phi$ is indeed shared by all wind conditions, and the linear weight $a$ contains the wind-specific information. Finally, note that $\phi$ with 0 training epoch reflects random features, which cannot decouple different wind conditions as cleanly as the trained representation $\phi$. Similarly, as shown in \cref{fig:training-tsne}, if we ignore the adversarial regularization term (by setting $\alpha=0$), different $a^*$ vectors in different conditions are less disentangled, which indicates that the learned representation might be less robust and explainable. 
For more discussions about $\alpha$ we refer to the supplementary materials (\ref{sec:supp-learning-hyperparameters}).

\subsection*{Robust Adaptive Controller Design}
\label{sec:control}
During the offline meta-training process, a least-squares fit is used to to find a set of parameters $a$ that minimizes the force prediction error for each data batch. 
However, during the online control phase, we are ultimately interested in minimizing the position tracking error and we can improve the adaptation using a more sophisticated update law.
Thus, in this section, we propose a more sophisticated adaptation law for the linear coefficients based upon a Kalman-filter estimator. This formulation results in automatic gain tuning for the update law, which allows the controller to quickly estimate parameters with large uncertainty. We further boost this estimator into a composite adaptation law, that is the parameter update depends both on the prediction error in the dynamics model as well as on the tracking error, as illustrated in \cref{fig:block-diagram}. This allows the system to quickly identify and adapt to new wind conditions without requiring persistent excitation. 
In turn, this enables online adaptation of the high dimensional learned models from \DAML.

Our online adaptive control algorithm can be summarized by the following control law, adaptation law, and covariance update equations, respectively.
\begin{align}
    \label{eq:control-law-our}
    \unf &= \underbrace{M(q) \ddot{q_r} + C(q, \dot q) \dot q_r + g(q)}_{\text{nominal model feedforward terms}} \underbrace{- K s}_{\text{PD feedback}} \underbrace{- \phi(q, \dot q) \hat a}_{\text{learning-based feedforward}}
    \\
    \label{eq:adaptation-law-a}
    \dot{\hat{a}} &= \underbrace{-\lambda \hat{a}}_{\text{regularization term}} \underbrace{- P \phi^\top R^{-1} (\phi \hat{a} - y)}_{\text{prediction error term}} \underbrace{+P \phi^\top s}_{\text{tracking error term}} \\
    \label{eq:adaptation-law-P}
    \dot{P} &= - 2 \lambda P + Q - P \phi ^\top R^{-1} \phi P
\end{align}
where $\unf$ is the control law, $\dot{\hat{a}}$ is the online linear-parameter update, $P$ is a covariance-like matrix used for automatic gain tuning, $s=\dot{\tilde{q}}+\Lambda\tilde{q}$ is the composite tracking error, \rev{$y$ is the measured aerodynamic residual force with measurement noise $\epsilon$}, and $K$, $\Lambda$, $R$, $Q$, and $\lambda$ are gains. The structure of this control law is illustrated in \cref{fig:block-diagram}. \Cref{fig:block-diagram} also shows further quadrotor specific details for the implementation of our method, including the kinematics block, where the desired thrust and attitude are determined from the desired force from \cref{eq:control-law-our}. These blocks are discussed further in the ``\nameref{sec:implementation}" section.

In the next section, we will first introduce the baseline control laws, $\bar{u}$ and $u_\text{NL}$.  Then we discuss our control law $\unf$ in detail. Note that $\unf$ not only depends on the desired trajectory, but also requires the learned representation $\phi$ and the linear parameter $\hat{a}$ (an estimation of $a$). The composite adaptation algorithm for $\hat{a}$ is discussed in the following section.

In terms of theoretical guarantees, the control law and adaptation law have been designed so that the closed-loop behavior of the system is robust to imperfect learning and time-varying wind conditions. Specifically, we define $d(t)$ as the representation error: $f = \phi \cdot a + d(t)$, and our theory shows that the robot tracking error exponentially converges to an error ball whose size is proportional to $\|d(t)\rev{+\epsilon}\|$ (i.e., the learning error and measurement noise) and $\|\dot{a}\|$ (i.e., how fast the wind condition changes). Later in this section we formalize these claims with the main stability theorem and present a complete proof in the supplementary materials.  

\paragraph{Nonlinear Control Law}
We start by defining some notation. The composite velocity tracking error term $s$ and the reference velocity $\dot q_r$ are defined such that 
\begin{align}
    \label{eq:s-dynamics}
    s = \dot q - \dot q_r = \dot{\tilde{q}} + \Lambda \tilde{q}
\end{align}
where $\tilde{q} = q - q_d$ is the position tracking error and $\Lambda$ is a positive definite gain matrix. Note when $s$ exponentially converges to an error ball around $0$, $q$ will exponentially converge to a proportionate error ball around the desired trajectory $q_d(t)$ (see \ref{sec:supp-stability}). Formulating our control law in terms of the composite velocity error $s$ simplifies the analysis and gain tuning without loss of rigor. 

The baseline nonlinear (NL) control law using PID feedback is defined as
\begin{equation}
    \label{eq:control-law-NL}
    u_\text{NL} = \underbrace{M(q) \ddot{q_r} + C(q, \dot q) \dot q_r + g(q)}_{\text{nonlinear feedforward terms}} \underbrace{- K s - K_I\int s dt}_{\text{PID feedback}}.
\end{equation}
where $K$ and $K_I$ are positive definite control gain matrices. Note a \rev{standard PID controller typically only includes the PI feedback on position error, D feedback on velocity, and gravity compensation. This only leads to} local exponential stability about a fixed point, but it is often sufficient for gentle tasks such as a UAV hovering and slow trajectories in static wind conditions. \rev{In contrast, this nonlinear controller includes feedback on velocity error and feedforward terms to account for known dynamics and desired acceleration, which allows good tracking of dynamic trajectories in the presence of nonlinearities (e.g., $M(q)$ and $C(q,\dot q)$ are nonconstant in attitude control)}. However, this control law only compensates for changing wind conditions and unmodeled dynamics through an integral term, which is slow to react \rev{to changes in the unmodelled dynamics and disturbance forces.}

Our method improves the controller by predicting the unmodeled dynamics and disturbance forces, and, indeed, in \cref{tab:flight-performance} we see a substantial improvement gained by using our learning method. Given the learned representation of the residual dynamics, $\phi(q, \dot q)$, and the parameter estimate $\hat a$, we replace the integral term with the learned force term, $\hat{f}=\phi \hat{a}$, resulting in our control law in \cref{eq:control-law-our}. 
\rev{\nf~uses $\phi$ trained using \DAML~on a dataset collected with the same drone. \nftransfer~uses $\phi$ trained using \DAML~on a dataset collected with a different drone, the Intel Aero RTF drone. \nfconstant~does not use any learning but instead uses $\phi=I$ and is included to demonstrate that the main advantage of our method comes from the incorporation of learning. The learning based methods, \nf~and \nftransfer, outperform \nfconstant~because the compact learned representation can effectively and quickly predict the aerodynamic disturbances online in \cref{fig:trajectory}. This comparison is further discussed in the supplementary materials (\ref{supp:prediction-performance}).}

\paragraph{Composite Adaptation Law}

We define an adaptation law that combines a tracking error update term, a prediction error update term, and a regularization term in \cref{eq:adaptation-law-a,eq:adaptation-law-P},
where $y$ is a noisy measurement of $f$, $\lambda$ is a damping gain, $P$ is a covariance matrix which evolves according to \cref{eq:adaptation-law-P}, and $Q$ and $R$ are two positive definite gain matrices. Some readers may note that the regularization term, prediction error term, and covariance update, when taken alone, are in the form of a Kalman-Bucy filter. This Kalman-Bucy filter can be derived as the optimal estimator that minimizes the variance of the parameter error \cite{kalman_new_1961}.
The Kalman-Bucy filter perspective provides intuition for tuning the adaptive controller: the damping gain $\lambda$ corresponds to how quickly the environment returns to the nominal conditions, $Q$ corresponds to how quickly the environment changes, and $R$ corresponds to the combined representation error $d$ and measurement noise for $y$. More discussion on the gain tuning process is included in \ref{sec:supp-gain-tuning}. However, naively combining this parameter estimator with the controller can lead to instabilities in the closed-loop system behavior unless extra care is taken in constraining the learned model and tuning the gains. Thus, we have designed our adaptation law to include a tracking error term, making \cref{eq:adaptation-law-a} a composite adaptation law, guaranteeing stability of the closed-loop system (see \Cref{thm:stability-proof}), and in turn simplifying the gain tuning process. The regularization term allows the stability result to be independent of the persistent excitation of the learned model $\phi$, which is particularly relevant when using high-dimensional learned representations. The adaptation gain and covariance matrix, $P$, acts as automatic gain tuning for the adaptive controller, which allows the controller to quickly adapt to when a new mode in the learned model is excited. 

\paragraph{Stability and Robustness Guarantees} 
First we formally define the representation error $d(t)$, as the difference between the unknown dynamics $f(q,\dot{q},w)$ and the best linear weight vector $a$ given the learned representation $\phi(q,\dot{q})$, namely, $d(t) = f(q,\dot{q},w) - \phi(q,\dot{q})a(w)$. The measurement noise for the measured residual force is a bounded function $\epsilon(t)$ such that $y(t) = f(t) + \epsilon(t)$. If the environment conditions are changing, we consider the case that $\dot{a}\neq 0$. This leads to the following stability theorem.
\begin{theorem}
\label{thm:stability-proof}
    If we assume that the desired trajectory has bounded derivatives and 
    the system evolves according to the dynamics in \cref{eq:open-loop-dynamics}, the control law \cref{eq:control-law-our}, and the adaptation law \cref{eq:adaptation-law-a,eq:adaptation-law-P}, then the position tracking error exponentially converges to the ball
    \begin{equation}
        \label{eq:exponential-bound}
        \lim_{t\rightarrow\infty}\|\tilde{q}\| \leq \sup_t \big[ C_1\|d(t)\| + C_2\|\rev{\epsilon(t)}\| +  C_3\left( \lambda\|a(t)\| + \|\dot{a}(t)\|\right) \big],
    \end{equation} 
    where $C_1$, $C_2$, and $C_3$ are three bounded constants depending on $\phi, R, Q, K,\Lambda, M$ and $\lambda$.
\end{theorem}

\subsection*{Implementation Details}
\label{sec:implementation}

\paragraph{Quadrotor Dynamics}
Now we introduce the quadrotor dynamics. Consider states given by global position, $p \in \mathbb{R}^3$, velocity $v \in \mathbb{R}^3$, attitude rotation matrix $R \in \mathrm{SO}(3)$, and body angular velocity $\omega \in \mathbb{R}^3$. Then dynamics of a quadrotor are
\begin{subequations}
\begin{align}
\dot{p} &= v, &  
m\dot{v} &= mg + Rf_u + f,\label{eq:pos-dynamics} \\ 
\dot{R}&=RS(\omega), & 
J\dot{\omega} &= J \omega \times \omega  + \tau_u,
\label{eq:att-dynamics}
\end{align}
\label{eq:quadrotor-dynamics}
\end{subequations}
where $m$ is the mass, $J$ is the inertia matrix of the quadrotor, $S(\cdot)$ is the skew-symmetric mapping, $g$ is the gravity vector, $f_u = [0, 0, T]^\top$ and $\tau_u = [\tau_x, \tau_y, \tau_z]^\top$ are the total thrust and body torques from four rotors predicted by the nominal model, and $f = [f_x, f_y, f_z]^\top$ are forces resulting from unmodelled aerodynamic effects due to varying wind conditions.

We cast the position dynamics in \cref{eq:pos-dynamics} into the form of \cref{eq:open-loop-dynamics}, by taking $M(q) = m I$, $C(q, \dot q) \equiv 0$, and $u = R f_u$. Note that the quadrotor attitude dynamics \cref{eq:att-dynamics} is also a special case of \cref{eq:open-loop-dynamics} \cite{slotine1991applied,murray_mathematical_2017}, and thus our method can be extended to attitude control. We implement our method in the position control loop, that is we use our method to compute a desired force $u_d$.  Then the desired force is decomposed into the desired attitude $R_d$ and the desired thrust $T_d$ using kinematics (see \cref{fig:block-diagram}). Then the desired attitude and thrust are sent to the onboard PX4 flight controller.

\paragraph{Neural Network Architectures and Training Details}
In practice, we found that in addition to the drone velocity $v$, the aerodynamic effects also depend on the drone attitude and the rotor rotation speed. To that end, the input state $x$ to the deep neural network $\phi$ is a 11-d vector, consisting of a the \revB{drone} velocity (3-d), \revB{the drone attitude represented as} a quaternion (4-d), and \revB{the rotor speed commands as} a pulse width modulation (PWM) signal (4-d) (see \cref{fig:block-diagram,fig:training-data}). The DNN $\phi$ has four fully-connected hidden layers, with an architecture $11\rightarrow50\rightarrow60\rightarrow50\rightarrow4$ and Rectified Linear Units (ReLU) activation. We found that the three components of the wind-effect force, $f_x,f_y,f_z$, are highly correlated and sharing common features, so we use $\phi$ as the basis function for all the component. Therefore, the wind-effect force $f$ is approximated by
\begin{equation}
    f \approx \begin{bmatrix}
    \phi(x) & 0 & 0 \\
    0 & \phi(x) & 0 \\
    0 & 0 & \phi(x)
    \end{bmatrix}
    \begin{bmatrix}
    a_x \\ a_y \\ a_z 
    \end{bmatrix},
\end{equation}
where $a_x,a_y,a_z\in\B{R}^4$ are the linear coefficients for each component of the wind-effect force. We followed Algorithm \ref{alg:DIML} to train $\phi$ in PyTorch, which is an open source deep learning framework. We refer to the supplementary material for hyperparameter details (\ref{sec:supp-learning-hyperparameters}).

Note that we explicitly include the PWM as an input to the $\phi$ network. The PWM information is a function of $u=Rf_u$, which makes the controller law (e.g., \cref{eq:control-law-our}) non-affine in $u$. We solve this issue by using the PWM from the last time step as an input to $\phi$, to compute the desired force $u_d$ at the current time step. Because we train $\phi$ using spectral normalization (see Algorithm \ref{alg:DIML}), this method is stable and guaranteed to converge to a fixed point, as discussed in \cite{shi2019neural}.

\paragraph{\rev{Controller Implementation}}
\rev{
For experiments, we implemented a discrete form of the \nf~controllers, given in \ref{sec:supp-discrete}. 
For INDI, we implemented the position and acceleration controller from Sections III.A and III.B in \cite{tal_accurate_2021}. For $\LOne$ adaptive control, we followed the adaptation law first presented in \cite{pravitra_L1-adaptive_2020} and used in \cite{hanover_performance_2021} and augment the nonlinear baseline control with $\hat{f} = - u_{\LOne}$.
}

\section*{SUPPLEMENTARY MATERIALS}
\ref{sec:supp-drone-configuration}. Drone Configuration Details \\
\ref{sec:supp-learning-expressiveness}. The expressiveness of the learning architecture \\
\ref{sec:supp-learning-hyperparameters}. Hyperparameters for \DAML~\rev{and the interpretation} \\
\ref{sec:supp-discrete}. Discrete version of the proposed controller \\
\ref{sec:supp-stability}. Stability and robustness formal guarantees and proof \\
\ref{sec:supp-gain-tuning}. Gain tuning \\
\rev{\ref{supp:prediction-performance}. Force prediction performance} \\
\ref{sec:supp-localization-error}. Localization error analysis \\
\Cref{fig:learning_curve}. Training and Validation Loss\\
\Cref{fig:domain-shift}. Importance of domain-invariant representation\\
\Cref{fig:force-prediction}. Measured residual force versus adaptive control augmentation \\
\Cref{fig:localization-good}. Localization inconsistency \\
\Cref{tab:drone_configuration}. Drone configuration details \\
\cref{tab:hardware-comparison}. Hardware comparison\\
\Cref{tab:learning_hyperparameter}. Hyperparameters used in DAIML

\renewcommand{\refname}{REFERENCES}
\bibliography{scibib}

\begin{thebibliography}{60}
\providecommand{\natexlab}[1]{#1}
\providecommand{\url}[1]{\texttt{#1}}
\expandafter\ifx\csname urlstyle\endcsname\relax
  \providecommand{\doi}[1]{doi: #1}\else
  \providecommand{\doi}{doi: \begingroup \urlstyle{rm}\Url}\fi

\bibitem[Foehn et~al.(2021)Foehn, Romero, and Scaramuzza]{foehn2021time}
Philipp Foehn, Angel Romero, and Davide Scaramuzza.
\newblock Time-optimal planning for quadrotor waypoint flight.
\newblock \emph{Science Robotics}, 6\penalty0 (56), 2021.

\bibitem[Faessler et~al.(2018)Faessler, Franchi, and
  Scaramuzza]{faessler_differential_2018}
Matthias Faessler, Antonio Franchi, and Davide Scaramuzza.
\newblock Differential {Flatness} of {Quadrotor} {Dynamics} {Subject} to
  {Rotor} {Drag} for {Accurate} {Tracking} of {High}-{Speed} {Trajectories}.
\newblock \emph{IEEE Robotics and Automation Letters}, 3\penalty0 (2):\penalty0
  620--626, April 2018.
\newblock ISSN 2377-3766.
\newblock \doi{10.1109/LRA.2017.2776353}.

\bibitem[Ventura~Diaz and Yoon(2018)]{ventura2018high}
Patricia Ventura~Diaz and Steven Yoon.
\newblock High-fidelity computational aerodynamics of multi-rotor unmanned
  aerial vehicles.
\newblock In \emph{2018 AIAA Aerospace Sciences Meeting}, page 1266, 2018.

\bibitem[Tal and Karaman(2021)]{tal_accurate_2021}
Ezra Tal and Sertac Karaman.
\newblock Accurate {Tracking} of {Aggressive} {Quadrotor} {Trajectories}
  {Using} {Incremental} {Nonlinear} {Dynamic} {Inversion} and {Differential}
  {Flatness}.
\newblock \emph{IEEE Transactions on Control Systems Technology}, 29\penalty0
  (3):\penalty0 1203--1218, May 2021.
\newblock ISSN 1558-0865.
\newblock \doi{10.1109/TCST.2020.3001117}.

\bibitem[Mallikarjunan et~al.(2012)Mallikarjunan, Nesbitt, Kharisov, Xargay,
  Hovakimyan, and Cao]{mallikarjunan_l1_2012}
Srinath Mallikarjunan, Bill Nesbitt, Evgeny Kharisov, Enric Xargay, Naira
  Hovakimyan, and Chengyu Cao.
\newblock L1 {Adaptive} {Controller} for {Attitude} {Control} of {Multirotors}.
\newblock In \emph{{AIAA} {Guidance}, {Navigation}, and {Control}
  {Conference}}, Minneapolis, Minnesota, August 2012. American Institute of
  Aeronautics and Astronautics.
\newblock ISBN 978-1-60086-938-9.
\newblock \doi{10.2514/6.2012-4831}.
\newblock URL \url{https://arc.aiaa.org/doi/10.2514/6.2012-4831}.

\bibitem[Pravitra et~al.(2020)Pravitra, Ackerman, Cao, Hovakimyan, and
  Theodorou]{pravitra_L1-adaptive_2020}
Jintasit Pravitra, Kasey~A. Ackerman, Chengyu Cao, Naira Hovakimyan, and
  Evangelos~A. Theodorou.
\newblock L1-{Adaptive} {MPPI} {Architecture} for {Robust} and {Agile}
  {Control} of {Multirotors}.
\newblock In \emph{2020 {IEEE}/{RSJ} {International} {Conference} on
  {Intelligent} {Robots} and {Systems} ({IROS})}, pages 7661--7666, October
  2020.
\newblock \doi{10.1109/IROS45743.2020.9341154}.
\newblock ISSN: 2153-0866.

\bibitem[Hanover et~al.(2021)Hanover, Foehn, Sun, Kaufmann, and
  Scaramuzza]{hanover_performance_2021}
Drew Hanover, Philipp Foehn, Sihao Sun, Elia Kaufmann, and Davide Scaramuzza.
\newblock Performance, precision, and payloads: Adaptive nonlinear mpc for
  quadrotors.
\newblock \emph{IEEE Robotics and Automation Letters}, 7\penalty0 (2):\penalty0
  690--697, 2021.

\bibitem[Shi et~al.(2019)Shi, Shi, O’Connell, Yu, Azizzadenesheli,
  Anandkumar, Yue, and Chung]{shi2019neural}
Guanya Shi, Xichen Shi, Michael O’Connell, Rose Yu, Kamyar Azizzadenesheli,
  Animashree Anandkumar, Yisong Yue, and Soon-Jo Chung.
\newblock Neural lander: Stable drone landing control using learned dynamics.
\newblock In \emph{2019 International Conference on Robotics and Automation
  (ICRA)}, pages 9784--9790. IEEE, 2019.

\bibitem[Shi et~al.(2020{\natexlab{a}})Shi, H{\"o}nig, Yue, and
  Chung]{shi2020neural}
Guanya Shi, Wolfgang H{\"o}nig, Yisong Yue, and Soon-Jo Chung.
\newblock Neural-swarm: Decentralized close-proximity multirotor control using
  learned interactions.
\newblock In \emph{2020 IEEE International Conference on Robotics and
  Automation (ICRA)}, pages 3241--3247. IEEE, 2020{\natexlab{a}}.

\bibitem[Shi et~al.(2021{\natexlab{a}})Shi, H{\"o}nig, Shi, Yue, and
  Chung]{shi2021neural}
Guanya Shi, Wolfgang H{\"o}nig, Xichen Shi, Yisong Yue, and Soon-Jo Chung.
\newblock Neural-swarm2: Planning and control of heterogeneous multirotor
  swarms using learned interactions.
\newblock \emph{IEEE Transactions on Robotics}, 2021{\natexlab{a}}.

\bibitem[Torrente et~al.(2021)Torrente, Kaufmann, F{\"o}hn, and
  Scaramuzza]{torrente2021data}
Guillem Torrente, Elia Kaufmann, Philipp F{\"o}hn, and Davide Scaramuzza.
\newblock Data-driven mpc for quadrotors.
\newblock \emph{IEEE Robotics and Automation Letters}, 6\penalty0 (2):\penalty0
  3769--3776, 2021.

\bibitem[Bartlett et~al.(2017)Bartlett, Foster, and
  Telgarsky]{bartlett2017spectrally}
Peter~L Bartlett, Dylan~J Foster, and Matus~J Telgarsky.
\newblock Spectrally-normalized margin bounds for neural networks.
\newblock \emph{Advances in Neural Information Processing Systems},
  30:\penalty0 6240--6249, 2017.

\bibitem[Slotine and Li(1991)]{slotine1991applied}
J.-J.~E. Slotine and Weiping Li.
\newblock \emph{Applied nonlinear control}.
\newblock Prentice Hall, Englewood Cliffs, N.J, 1991.
\newblock ISBN 978-0-13-040890-7.

\bibitem[Slotine and Li(1989)]{slotine_composite_1989}
Jean-Jacques~E. Slotine and Weiping Li.
\newblock Composite adaptive control of robot manipulators.
\newblock \emph{Automatica}, 25\penalty0 (4):\penalty0 509--519, July 1989.
\newblock ISSN 0005-1098.
\newblock \doi{10.1016/0005-1098(89)90094-0}.
\newblock URL
  \url{https://www.sciencedirect.com/science/article/pii/0005109889900940}.

\bibitem[Meier et~al.(2012)Meier, Tanskanen, Heng, Lee, Fraundorfer, and
  Pollefeys]{meier_pixhawk_2012}
Lorenz Meier, Petri Tanskanen, Lionel Heng, Gim~Hee Lee, Friedrich Fraundorfer,
  and Marc Pollefeys.
\newblock {PIXHAWK} -{A} micro aerial vehicle design for autonomous flight
  using onboard computer vision.
\newblock \emph{Autonomous Robots}, 33\penalty0 (1-2):\penalty0 21--39, August
  2012.
\newblock ISSN 0929-5593.
\newblock \doi{10.1007/s10514-012-9281-4}.
\newblock URL
  \url{https://graz.pure.elsevier.com/en/publications/pixhawk-a-micro-aerial-vehicle-design-for-autonomous-flight-using}.
\newblock Publisher: Springer Science + Business Media.

\bibitem[Mellinger and Kumar(2011)]{mellinger_minimum_2011}
Daniel Mellinger and Vijay Kumar.
\newblock Minimum snap trajectory generation and control for quadrotors.
\newblock In \emph{2011 {IEEE} {International} {Conference} on {Robotics} and
  {Automation}}, pages 2520--2525, May 2011.
\newblock \doi{10.1109/ICRA.2011.5980409}.

\bibitem[Ioannou and Sun(1996)]{ioannou1996robust}
Petros~A Ioannou and Jing Sun.
\newblock \emph{Robust adaptive control}, volume~1.
\newblock Prentice-Hall Upper Saddle River, NJ, 1996.

\bibitem[Krstic et~al.(1995)Krstic, Kokotovic, and
  Kanellakopoulos]{krstic1995nonlinear}
Miroslav Krstic, Petar~V Kokotovic, and Ioannis Kanellakopoulos.
\newblock \emph{Nonlinear and adaptive control design}.
\newblock John Wiley \& Sons, Inc., 1995.

\bibitem[Narendra and Annaswamy(2012)]{narendra2012stable}
Kumpati~S Narendra and Anuradha~M Annaswamy.
\newblock \emph{Stable adaptive systems}.
\newblock Courier Corporation, 2012.

\bibitem[Farrell and Polycarpou(2006)]{farrell_adaptive_2006}
Jay~A. Farrell and Marios~M. Polycarpou.
\newblock \emph{Adaptive {Approximation} {Based} {Control}}.
\newblock John Wiley \& Sons, Ltd, 2006.
\newblock ISBN 978-0-471-78181-3.
\newblock \doi{10.1002/0471781819.fmatter}.
\newblock URL
  \url{https://onlinelibrary.wiley.com/doi/abs/10.1002/0471781819.fmatter}.
\newblock \_eprint:
  https://onlinelibrary.wiley.com/doi/pdf/10.1002/0471781819.fmatter.

\bibitem[Wise et~al.(2006)Wise, Lavretsky, and Hovakimyan]{wise2006adaptive}
Kevin~A Wise, Eugene Lavretsky, and Naira Hovakimyan.
\newblock Adaptive control of flight: theory, applications, and open problems.
\newblock In \emph{2006 American Control Conference}, 2006.

\bibitem[Shi et~al.(2020{\natexlab{b}})Shi, Spieler, Tang, Lupu, Tokumaru, and
  Chung]{shi_adaptive_2020}
Xichen Shi, Patrick Spieler, Ellande Tang, Elena-Sorina Lupu, Phillip Tokumaru,
  and Soon-Jo Chung.
\newblock Adaptive {Nonlinear} {Control} of {Fixed}-{Wing} {VTOL} with
  {Airflow} {Vector} {Sensing}.
\newblock In \emph{2020 {IEEE} {International} {Conference} on {Robotics} and
  {Automation} ({ICRA})}, pages 5321--5327, Paris, France, May
  2020{\natexlab{b}}. IEEE.
\newblock ISBN 978-1-72817-395-5.
\newblock \doi{10.1109/ICRA40945.2020.9197344}.
\newblock URL \url{https://ieeexplore.ieee.org/document/9197344/}.

\bibitem[Rahimi and Recht(2007)]{rahimi2007random}
Ali Rahimi and Benjamin Recht.
\newblock Random features for large-scale kernel machines.
\newblock In \emph{Proceedings of the 20th International Conference on Neural
  Information Processing Systems}, pages 1177--1184, 2007.

\bibitem[Lale et~al.(2021)Lale, Azizzadenesheli, Hassibi, and
  Anandkumar]{lale_model_2021}
Sahin Lale, Kamyar Azizzadenesheli, Babak Hassibi, and Anima Anandkumar.
\newblock Model {Learning} {Predictive} {Control} in {Nonlinear} {Dynamical}
  {Systems}.
\newblock In \emph{2021 60th {IEEE} {Conference} on {Decision} and {Control}
  ({CDC})}, pages 757--762, December 2021.
\newblock \doi{10.1109/CDC45484.2021.9683670}.
\newblock ISSN: 2576-2370.

\bibitem[Nakanishi et~al.(2002)Nakanishi, Farrell, and
  Schaal]{nakanishi_locally_2002}
J.~Nakanishi, J.A. Farrell, and S.~Schaal.
\newblock A locally weighted learning composite adaptive controller with
  structure adaptation.
\newblock In \emph{{IEEE}/{RSJ} {International} {Conference} on {Intelligent}
  {Robots} and {Systems}}, volume~1, pages 882--889 vol.1, September 2002.
\newblock \doi{10.1109/IRDS.2002.1041502}.

\bibitem[Chen and Khalil(1995)]{chen1995adaptive}
Fu-Chuang Chen and Hassan~K Khalil.
\newblock Adaptive control of a class of nonlinear discrete-time systems using
  neural networks.
\newblock \emph{IEEE Transactions on Automatic Control}, 40\penalty0
  (5):\penalty0 791--801, 1995.

\bibitem[Johnson and Calise(2003)]{johnson2003limited}
Eric~N Johnson and Anthony~J Calise.
\newblock Limited authority adaptive flight control for reusable launch
  vehicles.
\newblock \emph{Journal of Guidance, Control, and Dynamics}, 26\penalty0
  (6):\penalty0 906--913, 2003.

\bibitem[Narendra and Mukhopadhyay(1997)]{narendra1997adaptive}
Kumpati~S Narendra and Snehasis Mukhopadhyay.
\newblock Adaptive control using neural networks and approximate models.
\newblock \emph{IEEE Transactions on Neural Networks}, 8\penalty0 (3):\penalty0
  475--485, 1997.

\bibitem[Bisheban and Lee(2021)]{bisheban_geometric_2021}
Mahdis Bisheban and Taeyoung Lee.
\newblock Geometric {Adaptive} {Control} {With} {Neural} {Networks} for a
  {Quadrotor} in {Wind} {Fields}.
\newblock \emph{IEEE Transactions on Control Systems Technology}, 29\penalty0
  (4):\penalty0 1533--1548, July 2021.
\newblock ISSN 1558-0865.
\newblock \doi{10.1109/TCST.2020.3006184}.

\bibitem[LeCun et~al.(2015)LeCun, Bengio, and Hinton]{lecun2015deep}
Yann LeCun, Yoshua Bengio, and Geoffrey Hinton.
\newblock Deep learning.
\newblock \emph{nature}, 521\penalty0 (7553):\penalty0 436--444, 2015.

\bibitem[Finn et~al.(2017)Finn, Abbeel, and Levine]{finn2017model}
Chelsea Finn, Pieter Abbeel, and Sergey Levine.
\newblock Model-agnostic meta-learning for fast adaptation of deep networks.
\newblock In \emph{International Conference on Machine Learning}, pages
  1126--1135. PMLR, 2017.

\bibitem[Hospedales et~al.(2021)Hospedales, Antoniou, Micaelli, and
  Storkey]{meta-learning-survey}
Timothy~M Hospedales, Antreas Antoniou, Paul Micaelli, and Amos~J. Storkey.
\newblock Meta-learning in neural networks: A survey.
\newblock \emph{IEEE Transactions on Pattern Analysis and Machine
  Intelligence}, pages 1--1, 2021.
\newblock \doi{10.1109/TPAMI.2021.3079209}.

\bibitem[Shi et~al.(2021{\natexlab{b}})Shi, Azizzadenesheli, O'Connell, Chung,
  and Yue]{shi2021meta}
Guanya Shi, Kamyar Azizzadenesheli, Michael O'Connell, Soon-Jo Chung, and
  Yisong Yue.
\newblock Meta-adaptive nonlinear control: Theory and algorithms.
\newblock \emph{Advances in Neural Information Processing Systems}, 34,
  2021{\natexlab{b}}.

\bibitem[Nagabandi et~al.(2018)Nagabandi, Clavera, Liu, Fearing, Abbeel,
  Levine, and Finn]{nagabandi2018learning}
Anusha Nagabandi, Ignasi Clavera, Simin Liu, Ronald~S Fearing, Pieter Abbeel,
  Sergey Levine, and Chelsea Finn.
\newblock Learning to adapt in dynamic, real-world environments through
  meta-reinforcement learning.
\newblock \emph{arXiv preprint arXiv:1803.11347}, 2018.

\bibitem[Song et~al.(2020)Song, Yang, Choromanski, Caluwaerts, Gao, Finn, and
  Tan]{song2020rapidly}
Xingyou Song, Yuxiang Yang, Krzysztof Choromanski, Ken Caluwaerts, Wenbo Gao,
  Chelsea Finn, and Jie Tan.
\newblock Rapidly adaptable legged robots via evolutionary meta-learning.
\newblock In \emph{2020 IEEE/RSJ International Conference on Intelligent Robots
  and Systems (IROS)}, pages 3769--3776. IEEE, 2020.

\bibitem[Belkhale et~al.(2021)Belkhale, Li, Kahn, McAllister, Calandra, and
  Levine]{belkhale2021model}
Suneel Belkhale, Rachel Li, Gregory Kahn, Rowan McAllister, Roberto Calandra,
  and Sergey Levine.
\newblock Model-based meta-reinforcement learning for flight with suspended
  payloads.
\newblock \emph{IEEE Robotics and Automation Letters}, 6\penalty0 (2):\penalty0
  1471--1478, 2021.

\bibitem[McKinnon and Schoellig(2021)]{mckinnon2021meta}
Christopher~D McKinnon and Angela~P Schoellig.
\newblock Meta learning with paired forward and inverse models for efficient
  receding horizon control.
\newblock \emph{IEEE Robotics and Automation Letters}, 6\penalty0 (2):\penalty0
  3240--3247, 2021.

\bibitem[Clavera et~al.(2018)Clavera, Rothfuss, Schulman, Fujita, Asfour, and
  Abbeel]{clavera2018model}
Ignasi Clavera, Jonas Rothfuss, John Schulman, Yasuhiro Fujita, Tamim Asfour,
  and Pieter Abbeel.
\newblock Model-based reinforcement learning via meta-policy optimization.
\newblock In \emph{Conference on Robot Learning}, pages 617--629. PMLR, 2018.

\bibitem[O'Connell et~al.(2021)O'Connell, Shi, Shi, and Chung]{o2021meta}
Michael O'Connell, Guanya Shi, Xichen Shi, and Soon-Jo Chung.
\newblock Meta-learning-based robust adaptive flight control under uncertain
  wind conditions.
\newblock \emph{arXiv preprint arXiv:2103.01932}, 2021.

\bibitem[Richards et~al.(2021)Richards, Azizan, Slotine, and
  Pavone]{richards2021adaptive}
Spencer~M Richards, Navid Azizan, Jean-Jacques~E Slotine, and Marco Pavone.
\newblock Adaptive-control-oriented meta-learning for nonlinear systems.
\newblock \emph{arXiv preprint arXiv:2103.04490}, 2021.

\bibitem[Peng et~al.(2021)Peng, Zhu, and Jiao]{peng2021linear}
Matt Peng, Banghua Zhu, and Jiantao Jiao.
\newblock Linear representation meta-reinforcement learning for instant
  adaptation.
\newblock \emph{arXiv preprint arXiv:2101.04750}, 2021.

\bibitem[Lee et~al.(2020)Lee, Hwangbo, Wellhausen, Koltun, and
  Hutter]{lee2020learning}
Joonho Lee, Jemin Hwangbo, Lorenz Wellhausen, Vladlen Koltun, and Marco Hutter.
\newblock Learning quadrupedal locomotion over challenging terrain.
\newblock \emph{Science robotics}, 5\penalty0 (47), 2020.

\bibitem[Tobin et~al.(2017)Tobin, Fong, Ray, Schneider, Zaremba, and
  Abbeel]{tobin2017domain}
Josh Tobin, Rachel Fong, Alex Ray, Jonas Schneider, Wojciech Zaremba, and
  Pieter Abbeel.
\newblock Domain randomization for transferring deep neural networks from
  simulation to the real world.
\newblock In \emph{2017 IEEE/RSJ International Conference on Intelligent Robots
  and Systems (IROS)}, pages 23--30. IEEE, 2017.

\bibitem[Ramos et~al.(2019)Ramos, Possas, and Fox]{ramos2019bayessim}
Fabio Ramos, Rafael~Carvalhaes Possas, and Dieter Fox.
\newblock Bayessim: adaptive domain randomization via probabilistic inference
  for robotics simulators.
\newblock \emph{arXiv preprint arXiv:1906.01728}, 2019.

\bibitem[Morgan et~al.(2016)Morgan, Subramanian, Chung, and
  Hadaegh]{morgan_swarm_2015}
Daniel Morgan, Giri~P Subramanian, Soon-Jo Chung, and Fred~Y Hadaegh.
\newblock Swarm assignment and trajectory optimization using variable-swarm,
  distributed auction assignment and sequential convex programming.
\newblock \emph{The International Journal of Robotics Research}, 35\penalty0
  (10):\penalty0 1261--1285, 2016.

\bibitem[Shi et~al.(2018)Shi, Kim, Rahili, and Chung]{shi2018nonlinear}
Xichen Shi, Kyunam Kim, Salar Rahili, and Soon-Jo Chung.
\newblock Nonlinear control of autonomous flying cars with wings and
  distributed electric propulsion.
\newblock In \emph{2018 IEEE Conference on Decision and Control (CDC)}, pages
  5326--5333. IEEE, 2018.

\bibitem[Nakka et~al.(2020)Nakka, Liu, Shi, Anandkumar, Yue, and
  Chung]{nakka2020chance}
Yashwanth~Kumar Nakka, Anqi Liu, Guanya Shi, Anima Anandkumar, Yisong Yue, and
  Soon-Jo Chung.
\newblock Chance-constrained trajectory optimization for safe exploration and
  learning of nonlinear systems.
\newblock \emph{IEEE Robotics and Automation Letters}, 6\penalty0 (2):\penalty0
  389--396, 2020.

\bibitem[Loquercio et~al.(2021)Loquercio, Kaufmann, Ranftl, Müller, Koltun,
  and Scaramuzza]{doi:10.1126/scirobotics.abg5810}
Antonio Loquercio, Elia Kaufmann, René Ranftl, Matthias Müller, Vladlen
  Koltun, and Davide Scaramuzza.
\newblock Learning high-speed flight in the wild.
\newblock \emph{Science Robotics}, 6\penalty0 (59):\penalty0 eabg5810, 2021.
\newblock \doi{10.1126/scirobotics.abg5810}.
\newblock URL
  \url{https://www.science.org/doi/abs/10.1126/scirobotics.abg5810}.

\bibitem[Liu et~al.(2020)Liu, Shi, Chung, Anandkumar, and Yue]{liu2020robust}
Anqi Liu, Guanya Shi, Soon-Jo Chung, Anima Anandkumar, and Yisong Yue.
\newblock Robust regression for safe exploration in control.
\newblock In \emph{Learning for Dynamics and Control}, pages 608--619. PMLR,
  2020.

\bibitem[Kim et~al.(2021)Kim, Spieler, Lupu, Ramezani, and
  Chung]{doi:10.1126/scirobotics.abf8136}
Kyunam Kim, Patrick Spieler, Elena-Sorina Lupu, Alireza Ramezani, and Soon-Jo
  Chung.
\newblock A bipedal walking robot that can fly, slackline, and skateboard.
\newblock \emph{Science Robotics}, 6\penalty0 (59):\penalty0 eabf8136, 2021.
\newblock \doi{10.1126/scirobotics.abf8136}.

\bibitem[Ganin et~al.(2016)Ganin, Ustinova, Ajakan, Germain, Larochelle,
  Laviolette, Marchand, and Lempitsky]{ganin2016domain}
Yaroslav Ganin, Evgeniya Ustinova, Hana Ajakan, Pascal Germain, Hugo
  Larochelle, Fran{\c{c}}ois Laviolette, Mario Marchand, and Victor Lempitsky.
\newblock Domain-adversarial training of neural networks.
\newblock \emph{The Journal of Machine Learning Research}, 17\penalty0
  (1):\penalty0 2096--2030, 2016.

\bibitem[Goodfellow et~al.(2014)Goodfellow, Pouget-Abadie, Mirza, Xu,
  Warde-Farley, Ozair, Courville, and Bengio]{goodfellow2014generative}
Ian Goodfellow, Jean Pouget-Abadie, Mehdi Mirza, Bing Xu, David Warde-Farley,
  Sherjil Ozair, Aaron Courville, and Yoshua Bengio.
\newblock Generative adversarial nets.
\newblock \emph{Advances in Neural Information Processing Systems}, 27, 2014.

\bibitem[Kalman and Bucy(1961)]{kalman_new_1961}
R.~E. Kalman and R.~S. Bucy.
\newblock New {Results} in {Linear} {Filtering} and {Prediction} {Theory}.
\newblock \emph{Journal of Basic Engineering}, 83\penalty0 (1):\penalty0
  95--108, March 1961.
\newblock ISSN 0021-9223.
\newblock \doi{10.1115/1.3658902}.
\newblock URL
  \url{https://asmedigitalcollection.asme.org/fluidsengineering/article/83/1/95/426820/New-Results-in-Linear-Filtering-and-Prediction}.

\bibitem[Murray et~al.(2017)Murray, Li, and Sastry]{murray_mathematical_2017}
Richard~M. Murray, Zexiang Li, and S.~Shankar Sastry.
\newblock \emph{A {Mathematical} {Introduction} to {Robotic} {Manipulation}}.
\newblock CRC Press, 1 edition, December 2017.
\newblock ISBN 978-1-315-13637-0.
\newblock \doi{10.1201/9781315136370}.
\newblock URL \url{https://www.taylorfrancis.com/books/9781351469791}.

\bibitem[Trefethen(2017)]{trefethen2017multivariate}
Lloyd Trefethen.
\newblock Multivariate polynomial approximation in the hypercube.
\newblock \emph{Proceedings of the American Mathematical Society}, 145\penalty0
  (11):\penalty0 4837--4844, 2017.

\bibitem[Yarotsky(2017)]{yarotsky2017error}
Dmitry Yarotsky.
\newblock Error bounds for approximations with deep relu networks.
\newblock \emph{Neural Networks}, 94:\penalty0 103--114, 2017.

\bibitem[Dieci and Eirola(1994)]{dieci_positive_1994}
Luca Dieci and Timo Eirola.
\newblock Positive definiteness in the numerical solution of {Riccati
  differential} equations.
\newblock \emph{Numerische Mathematik}, 67\penalty0 (3):\penalty0 303--313,
  April 1994.
\newblock ISSN 0945-3245.
\newblock \doi{10.1007/s002110050030}.
\newblock URL \url{https://doi.org/10.1007/s002110050030}.

\bibitem[Kalman(1960)]{kalman_new_1960}
R.~E. Kalman.
\newblock A {New} {Approach} to {Linear} {Filtering} and {Prediction}
  {Problems}.
\newblock \emph{Journal of Basic Engineering}, 82\penalty0 (1):\penalty0
  35--45, March 1960.
\newblock ISSN 0021-9223.
\newblock \doi{10.1115/1.3662552}.
\newblock URL \url{https://doi.org/10.1115/1.3662552}.

\bibitem[Khalil(2002)]{khalil_nonlinear_2002}
Hassan~K. Khalil.
\newblock \emph{Nonlinear {Systems}, 3rd {Edition}}.
\newblock Prentice Hall, 2002.
\newblock URL
  \url{https://www.pearson.com/content/one-dot-com/one-dot-com/us/en/higher-education/program.html}.

\bibitem[noa(2021)]{noauthor_multicopter_nodate}
Multicopter {PID} {Tuning} {Guide} {\textbar} {PX4} {User} {Guide}, 12 2021.
\newblock URL
  \url{https://docs.px4.io/master/en/config_mc/pid_tuning_guide_multicopter.html}.

\end{thebibliography}
\bibliographystyle{unsrtnat}

\section*{ACKNOWLEDGEMENTS}
\textbf{Acknowledgements:} A.A. is also affiliated with NVIDIA Corporation, and Y.Y. is also with associated Argo AI. K.A. is currently affiliated with Purdue University. We thank J. Burdick and J.-J. E. Slotine for their helpful discussions. We thank M. Anderson for help with configuring the quadrotor platform, and M. Anderson and P. Spieler for help with hardware troubleshooting. We also thank N. Badillo and L. Pabon Madrid for help in experiments.

\textbf{Funding:} This research was developed with funding from the Defense Advanced Research Projects Agency (DARPA). This research was also conducted in part with funding from Raytheon Technologies.
The views, opinions, and/or findings expressed are those of the authors and
should not be interpreted as representing the official views or policies of
the Department of Defense or the U.S. Government. The experiments reported in this article were conducted at Caltech's Center for Autonomous Systems and Technologies (CAST). \textbf{Author contributions:} (1) S.-J.C. and Y.Y. directed the research activities. (2) G.S. and M.O'C. designed and implemented the meta-learning algorithm under the guidance of Y.Y., K.A., A.A., S.-J.C., while the last-layer adaptation idea was started with a discussion by G.S., M.O'C., X.S., and S.-J.C. (3) M.O'C., G.S. designed and implemented the adaptive control algorithm with inputs from S.-J.C., X.S. (4) M.O'C., G.S. performed experiments and evaluated the results. (5) M.O'C. conducted the theoretical analysis of the meta-learning based adaptive controller with input from S.-J.C., G.S., X.S. (6) G.S. analyzed the learning algorithm with feedback from Y.Y., K.A., A.A., and S.-J.C. (7) G.S., M.O'C. made all the figures and videos with input from all the others. (8) All authors prepared the manuscript. 
\textbf{Competing interests:} The authors
declare that they have no competing interests. \textbf{Data and materials availability:} All data needed to evaluate the conclusions in the article are present in the article or in the Supplementary Materials.


%
%


\sectionfont{\MakeUppercase}
\section*{Supplementary Materials}

\setcounter{table}{0}
\renewcommand{\thetable}{S\arabic{table}}

\setcounter{subsection}{0}
\renewcommand{\thesubsection}{Section S\arabic{subsection}}

\setcounter{figure}{0}
\renewcommand{\thefigure}{S\arabic{figure}}

\subsection{Drone Configuration Details}
\label{sec:supp-drone-configuration}

\Cref{tab:drone_configuration} presents the configuration information of the custom built drone (fig.~\ref{fig:intro}(A)) and the Intel Aero drone. We use both drones for data collection and use the custom built drone exclusively for experiments.

\begin{table}[ht]
    \centering
    \begin{tabular}{c|cc}
    & Custom built drone & Intel Aero drone \\
    \hline
    Weight & \SI{2.53}{kg} & \SI{1.47}{kg} \\
    Thrust-to-weight ratio & 2.2 & 1.6 \\ 
    Rotor tilt angle & \SI{12}{\degree} front, \SI{10}{\degree} rear & \SI{0}{\degree} \\ 
    Diameter & \SI{85}{cm} wide, \SI{75}{cm} long & \SI{52}{cm} wide, \SI{52}{cm} long \\ 
    Configuration & Wide-X4 & X4 \\
    On-board computer & Raspberry Pi 4 & Intel Aero computing board (Atom x7 processor) \\
    Flight controller & Pixhawk 4 running PX4 & Aero Flight Controller running PX4 \\
    \end{tabular}
    \caption{\textbf{Drone configuration details}~Configurations of the custom built drone and the Intel Aero drone with propeller guards.}
    \label{tab:drone_configuration}
\end{table}

Precision tracking for drones often relies on specialized hardware and optimized vehicle design, whereas our method achieves precise tracking using improved dynamics prediction through online learning. 
Although most researchers report the numeric tracking error of their method, it can be difficult to disentangle the improvement of the controller resulting from the algorithmic advancement versus the improvement from specialized hardware. For example moment of inertia generally scales with the radius squared and the lever arm for the motors scales with the radius, so the attitude maneuverability roughly scales with the inverse of the vehicle radius. Similarly, high thrust to weight ratio provides more attitude control authority during high acceleration maneuvers. More powerful motors, electronic speed controllers, and batteries together allow faster motor response time further improving maneuverability. Thus, state-of-the-art (SOTA) tracking performance usually requires specialized hardware often used for racing drones, resulting in a vehicle with greater maneuverability than our platform, a higher thrust to weight ratio, and using high-rate controllers sometimes even including direct motor RPM control. In contrast, our custom drone is more representative of typical consumer drone hardware. A detailed comparison with the hardware from some recent work in agile flight control is provided in \cref{tab:hardware-comparison}.

\begin{table}[ht]
    \centering
    \begin{small}
    \begin{tabular}{c|c|c|c|c}
         & \nf~& INDI \cite{tal_accurate_2021} & Differentially flat                          & Gaussian Process \\
         &     &                              & linear drag \cite{faessler_differential_2018} & MPC \cite{torrente2021data}\\
         \hline
        Flight computer & Raspberry Pi 4 & -- & laptop & laptop \\
        Flight controller & Pixhawk 4 & STM32H7 &  Raceflight Revolt & ?\\
        && (\SI{400}{MHz}) &&\\
        Flight controller firmware & PX4 & custom & ? & ?\\
        Mass [\SI{}{kg}] & 2.53 & 0.609 & 0.610 & 0.8 \\
        Total width [\SI{}{cm}] & 85 & ? & ? & ? \\
        Propeller diameter [\SI{}{in}] & 11 & 5 & 6 & ? \\
        Motor Spacing [\SI{}{cm}] & 39* & 18 & -- & ?\\
        Thrust-to-weight ratio [-] & 2.2 & ? & 4 & 5 \\
        Motion capture frequency [\SI{}{Hz}] & 100 & 360 & 200 & 100 \\
        MPC control frequency [\SI{}{Hz}] & -- & -- & -- & 50 \\
        Position control frequency [\SI{}{Hz}] & 50 & ? & 55 & ?\\
        Attitude control frequency [\SI{}{Hz}] & $<$1000 & 2000 & 4000 & ? \\
        Motor speed feedback & No & Optical encoders  & No & No\\
        & & (\SI{5}{kHz}) & & \\
        \multicolumn{5}{l}{}\\
        \multicolumn{5}{l}{? indicates information not provided}\\
        \multicolumn{5}{l}{-- indicates information not applicable}\\
        \multicolumn{5}{l}{*  front to back}
    \end{tabular}
    \caption{\textbf{Hardware comparison} Hardware configuration comparison with other quadrotors that demonstrate state-of-the-art trajectory tracking. Direct comparisons of performance are difficult due to the varying configurations, controller tuning, and flight arenas. However, most methods require extremely maneuverable quadrotors and onboard/offboard computation power to achieve state-of-the-art performance, while \nf~achieves state-of-the-art performance on more standard hardware with all control running onboard.}
    \label{tab:hardware-comparison}
    \end{small}
\end{table}

\subsection{The Expressiveness of the Learning Architecture}
\label{sec:supp-learning-expressiveness}

In this section, we theoretically justify the decomposition $\BB f(\BB x,\BB w)\approx \phi(\BB x)\BB a(\BB w)$. In particularly, we prove that any analytic function $\bar{f}(\BB x,\BB w):[-1,1]^n\times[-1,1]^m\rightarrow\B{R}$ can be split into a $\BB w$-invariant part $\bar{\phi}(\BB x)$ and a $\BB w$-dependant part $\bar{\BB a}(\BB w)$ in the structure $\bar{\phi}(\BB x)\bar{\BB a}(\BB w)$ with arbitrary precision $\epsilon$, where $\bar{\phi}(\BB x)$ and $\bar{\BB a}(\BB w)$ are two polynomials. Further, the dimension of $\bar{\BB a}(\BB w)$ only scales polylogarithmically with $1/\epsilon$. 

We first introduce the following multivariate polynomial approximation lemma in the hypercube proved in \cite{trefethen2017multivariate}.
\begin{lemma}(Multivariate polynomial approximation in the hypercube)
\label{lemma:poly_appro}
Let $\bar{f}(\BB x,\BB w):[-1,1]^n\times[-1,1]^m\rightarrow\B{R}$ be a smooth function of $[\BB x,\BB w]\in[-1,1]^{n+m}$ for $n,m\geq 1$. Assume $\bar{f}(\BB x,\BB w)$ is analytic for all $[\BB x,\BB w]\in\B{C}^{n+m}$ with $\Re(x_1^2+\cdots+x_n^2+w_1^2+\cdots+w_m^2)\geq-t^2$ for some $t>0$, where $\Re(\cdot)$ denotes the real part of a complex number. Then $\bar{f}$ has a uniformly and absolutely convergent multivariate Chebyshev series 
$$\sum_{k_1=0}^\infty\cdots\sum_{k_n=0}^\infty\sum_{l_1=0}^\infty\cdots\sum_{l_m=0}^\infty
b_{k_1,\cdots,k_n,l_1,\cdots,l_m}T_{k_1}(x_1)\cdots T_{k_n}(x_n)T_{l_1}(w_1)\cdots T_{l_m}(w_m).
$$
Define $\BB s=[k_1,\cdots,k_n,l_1,\cdots,l_m]$. The multivariate Chebyshev coefficients satisfy the following exponential decay property:
$$b_{s}=O\left( (1+t)^{-\|\BB s\|_2} \right).$$
\end{lemma}

Note that this lemma shows that the truncated Chebyshev expansions
$$\mathcal{C}_p=\sum_{k_1=0}^p\cdots\sum_{k_n=0}^p\sum_{l_1=0}^p\cdots\sum_{l_m=0}^p
b_{k_1,\cdots,k_n,l_1,\cdots,l_m}T_{k_1}(x_1)\cdots T_{k_n}(x_n)T_{l_1}(w_1)\cdots T_{l_m}(w_m)$$
will converge to $\bar{f}$ with the rate $O((1+t)^{-p\sqrt{n+m}})$ for some $t>0$, i.e., $\sup_{[\BB x,\BB w]\in[-1,1]^{n+m}}\|\bar{f}(\BB x,\BB w)-\mathcal{C}_p(\BB x,\BB w)\|\leq O((1+t)^{-p\sqrt{n+m}})$. Finally we are ready to present the following representation theorem.

\begin{theorem}
\label{thm:representation}
$\bar{f}(\BB x,\BB w)$ is a function satisfying the assumptions in Lemma~\ref{lemma:poly_appro}. For any $\epsilon>0$, there exist $h\in\B{Z}^+$, and two Chebyshev polynomials $\bar{\phi}(\BB x):[-1,1]^n\rightarrow\B{R}^{1\times h}$ and  $\bar{\BB a}(\BB w):[-1,1]^m\rightarrow\B{R}^{h\times 1}$ such that 
$$\sup_{[\BB x,\BB w]\in[-1,1]^{n+m}}\|\bar{f}(\BB x,\BB w)-\bar{\phi}(\BB x) \bar{\BB a}(\BB w)\|\leq\epsilon$$
and $h=O((\log(1/\epsilon))^m)$.
\end{theorem}
\begin{proof}
First note that there exists $p=O\left(\frac{\log(1/\epsilon)}{\sqrt{n+m}}\right)$ such that $\sup_{[\BB x,\BB w]\in[-1,1]^{n+m}}\left\|\bar{f}(\BB x,\BB w)-\mathcal{C}_p(\BB x,\BB w)\right\|\leq\epsilon$. To simplify the notation, define
$$\begin{aligned}
g(\BB x,\BB k,\BB l) &= g(x_1,\cdots,x_n,k_1,\cdots,k_n,l_1,\cdots,l_m)=b_{k_1,\cdots,k_n,l_1,\cdots,l_m}T_{k_1}(x_1)\cdots T_{k_n}(x_n) \\
g(\BB w,\BB l) &= g(w_1,\cdots,w_m,l_1,\cdots,l_m) = T_{l_1}(w_1)\cdots T_{l_n}(w_m)
\end{aligned}$$
Then we have
$$\mathcal{C}_p(\BB x,\BB w) = \sum_{k_1,\cdots,k_n=0}^p\,\,\sum_{l_1,\cdots,l_m=0}^p g(\BB x,k_1,\cdots,k_n,l_1,\cdots,l_m)g(\BB w,l_1,\cdots,l_m)$$
Then we rewrite $\mathcal{C}_p$ as $\mathcal{C}_p(\BB x,\BB w)=\bar{\phi}(\BB x)\bar{\BB a}(\BB w)$:
\begin{equation*}
\begin{aligned}
\bar{\phi}(\BB x)^\top &= \begin{bmatrix}
\sum_{k_1,\cdots,k_n=0}^p\, g(\BB x,k_1,\cdots,k_n,\BB l=[0,0,\cdots,0]) \\
\sum_{k_1,\cdots,k_n=0}^p\, g(\BB x,k_1,\cdots,k_n,\BB l=[1,0,\cdots,0]) \\
\sum_{k_1,\cdots,k_n=0}^p\, g(\BB x,k_1,\cdots,k_n,\BB l=[2,0,\cdots,0]) \\
\vdots \\
\sum_{k_1,\cdots,k_n=0}^p\, g(\BB x,k_1,\cdots,k_n,\BB l=[p,p,\cdots,p])
\end{bmatrix},
\bar{a}(w) = \begin{bmatrix}
g(\BB w,\BB l=[0,0,\cdots,0]) \\
g(\BB w,\BB l=[1,0,\cdots,0]) \\
g(\BB w,\BB l=[2,0,\cdots,0]) \\
\vdots \\
g(\BB w,\BB l=[p,p,\cdots,p])
\end{bmatrix}
\end{aligned}
\end{equation*}
Note that the dimension of $\bar{\phi}(\BB x)$ and $\bar{\BB a}(\BB w)$ is
$$h=(p+1)^{m}=O\left(\left(1+\frac{\log(1/\epsilon)}{\sqrt{n+m}}\right)^m\right)=O\left((\log(1/\epsilon))^m\right)$$
\end{proof}
Note that Theorem \ref{thm:representation} can be generalized to vector-valued functions with bounded input space straightforwardly. Finally, since deep neural networks are universal approximators for polynomials \cite{yarotsky2017error}, Theorem \ref{thm:representation} immediately guarantees the expressiveness of our learning structure, i.e., $\phi(\BB x)\BB a(\BB w)$ can approximate $\BB f(\BB x,\BB w)$ with arbitrary precision, where $\phi(\BB x)$ is a deep neural network and $\ahat$ includes the linear coefficients for all the elements of $\BB f$. In experiments, we show that a four-layer neural network can efficiently learn an effective representation for the underlying unknown dynamics $\BB f(\BB x,\BB w)$.

\subsection{Hyperparameters for \DAML~and the Interpretation}
\label{sec:supp-learning-hyperparameters}

We implemented \DAML~(Algorithm \ref{alg:DIML}) using PyTorch, with hyperparameters reported in \cref{tab:learning_hyperparameter}. We iteratively tuned these hyperparameters by trial and error. We notice that the behavior of the learning algorithm is not sensitive to most of parameters in \cref{tab:learning_hyperparameter}. The training process is shown in fig.~\ref{fig:learning_curve}, where we present the $f$ loss curve on both training set and validation set using three random seeds. The $f$ loss is defined by $\sum_{i\in B}\|\BB y_k^{(i)}-\phi(\BB x_k^{(i)})\BB a^*\|^2$ (see Line \ref{alg:DIML:learning_phi} in Algorithm \ref{alg:DIML}), which reflects how well $\phi$ can approximate the unknown dynamics $\BB f(\BB x,\BB w)$. The validation set we considered is from the figure-8 trajectory tracking tasks using the PID and nonlinear baseline methods. Note that the training set consists of a very different set of trajectories (using random waypoint tracking, see Results)
, and this difference is for studying whether and when the learned model $\phi$ starts over-fitting during the training process.

\begin{figure}
    \centering
    \includegraphics[width=0.8\linewidth]{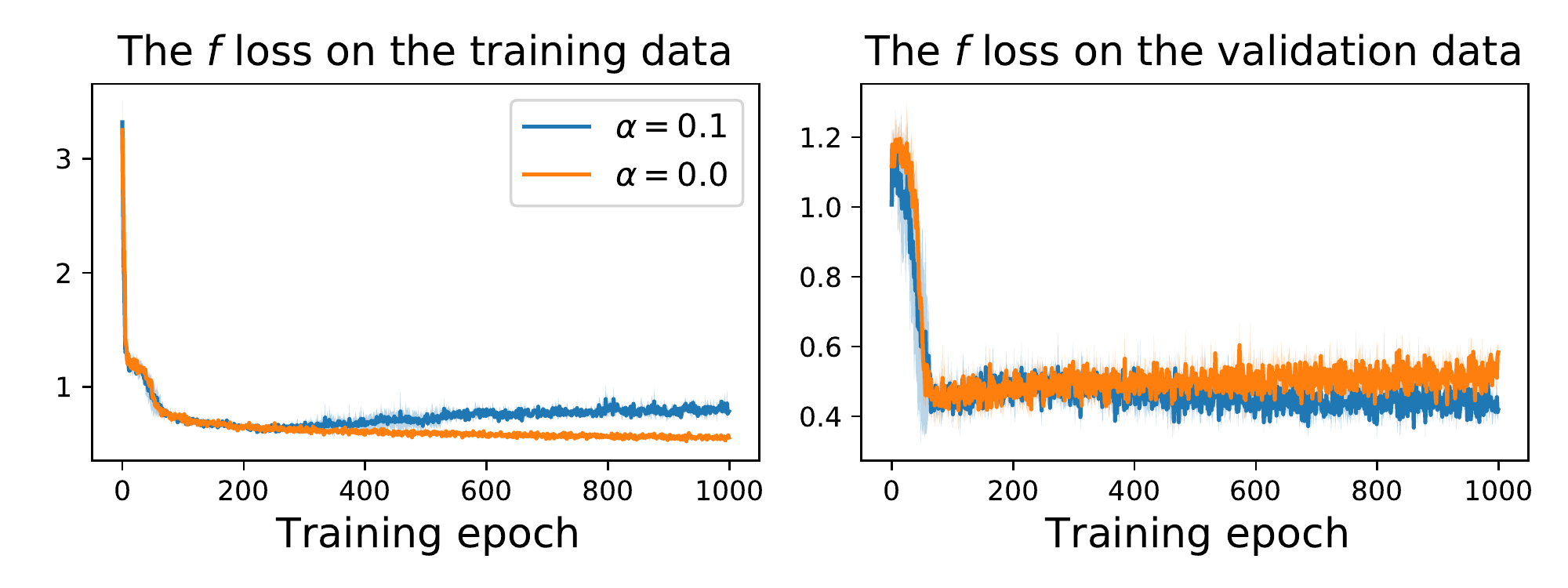}
    \caption{\textbf{Training and validation loss.} The evolution of the $f$ loss on the training data and validation data in the training process, from three random seeds. Both mean (the solid line) and standard deviation (in the shaded area) are presented. Training with the adversarial regularization term ($\alpha=0.1$) has similar behaviors as $\alpha=0$ (no regularization) in the early phase before 300 training epochs, except that it converges slightly faster. However, the regularization term effectively avoids over-fitting and has smaller error on the validation dataset after 300 training epochs.}
    \label{fig:learning_curve}
\end{figure}

\begin{table}
    \centering
    \begin{tabular}{c|c}
    Architecture of $\phi$ net & $11\rightarrow50\rightarrow60\rightarrow50\rightarrow4$ with ReLU activation functions \\
    Architecture of $h$ net & $4\rightarrow128\rightarrow6$ with ReLU activation functions \\
    Batch size of $B_a$ & 128 \\
    Batch size of $B$ & 256 \\
    Loss function for $h$ & Cross-entropy loss \\
    Learning rate for training $\phi$ & $0.0005$ \\
    Learning rate for training $h$ & $0.001$ \\
    Discriminator training frequency $\eta$ & $0.5$ \\
    Normalization constant $\gamma$ & $10$ \\
    The degree of regularization $\alpha$ & $0.1$ \\
    \end{tabular}
    \caption{Hyperparameters used in \DAML~(Algorithm \ref{alg:DIML}).}
    \label{tab:learning_hyperparameter}
\end{table}
 
We emphasize a few important parameters as follows. (i) The frequency $0<\eta\leq1$ is to control how often the discriminator $\BB h$ is updated. Note that $\eta=1$ corresponds to the case that $\phi$ and $h$ are both updated in each iteration. We use $\eta=0.5$ for training stability, which is also commonly used in training generative adversarial networks \cite{goodfellow2014generative}. (ii) The regularization parameter $\alpha\geq0$. Note that $\alpha=0$ corresponds to the non-adversarial meta-learning case which does not incorporate the adversarial regularization term in \cref{eq:optimization-both-loss}. From fig.~\ref{fig:learning_curve}, clearly a proper choice of $\alpha$ can effectively avoid over-fitting. Moreover, another benefit of having $\alpha>0$ is that the learned model is more explainable. As observed in fig.~\rev{fig:training-tsne}, $\alpha>0$ disentangles the linear coefficients $\BB a^*$ between wind conditions. However, if $\alpha$ is too high it may degrade the prediction performance, so we recommend using relatively small value for $\alpha$ such as $0.1$.

\rev{\paragraph{The importance of having a domain-invariant representation.}
We use the following example to illustrate the importance of having a domain-invariant representation $\phi(\BB x)$ for online adaptation. Suppose the data distribution in wind conditions 1 and 2 are $P_1(\BB x)$ and $P_2(\BB x)$, respectively, and they do not overlap. Ideally, we would hope these two conditions share an invariant representation and the latent variables are distinct ($\BB a^{(1)}$ and $\BB a^{(2)}$ in the first line in fig.~\ref{fig:domain-shift} shown below). However, because of the expressiveness of DNNs, $\phi$ may memorize $P_1$ and $P_2$ and learn two modes $\phi_1(\BB x)$ and $\phi_2(\BB x)$. In the second line in the following figure, $\phi_1$ and $\phi_2$ are triggered if $\BB x$ is in $P_1$ and $P_2$, respectively ($\mathbf{1}_{\BB x\in{P_1}}$ and $\mathbf{1}_{\BB x\in{P_2}}$ are indicator functions), such that the latent variable $\BB a$ is identical in both wind conditions. Such an overfitted $\phi$ is not robust and not generalizable: for example, if the drone flies to $P_1$ in wind condition 2, the wrong mode $\phi_1$ will be triggered.}

\begin{figure}[h]
    \centering
    \includegraphics[width=0.8\linewidth]{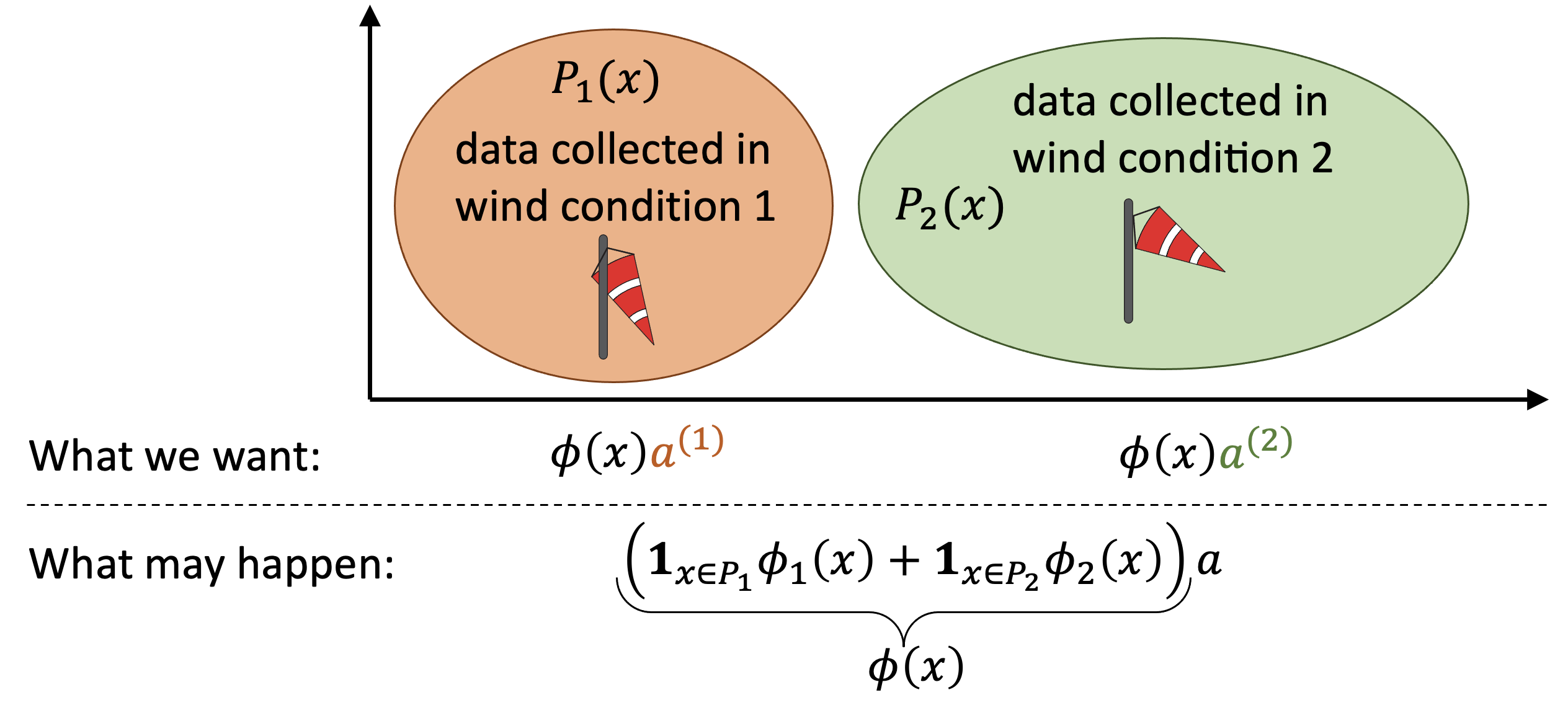}  
    \caption{\textbf{Importance of domain-invariant representation.}}
    \label{fig:domain-shift}
\end{figure}
\begin{center} 
\end{center}

\rev{The key idea to tackle this challenge is to encourage diversity in the latent space, which is why we introduced a discriminator in DAIML. Figure~\ref{fig:training-tsne} shows DAIML indeed makes the latent space much more disentangled.}

\subsection{Discrete Version of the Proposed Controller}
\label{sec:supp-discrete}

In practice, we implement \nf on a digital system, and therefore, we require a discrete version of the controller. The feedback control policy $\BB u$ remains the same as presented in the main body of this article. However, the adaptation law must be integrated and therefore we must be concerned with both the numerical accuracy and computation time of this integration, particularly for the covaraince matrix $P$. During the development of our algorithm, we observed that a naive one-step Euler integration of the continuous time adaptation law would sometimes result $P$ becoming non-positive-definite due to a large $\dot{P}$ magnitude and a coarse integration step size (see \cite{dieci_positive_1994} for more discussion on the positive definiteness of numerical integration of the differential Riccati equation). To avoid this issue, we instead implemented the adaptation law in two discrete steps, a propagation and an update step, summarized as below. We denote the time at step $k$ as $t_k$, the value of a parameter before the update step but after the propagation step with a subscript $t_k^-$, and the value after both the propagation and update step with a subscript $t_k^+$. The value used in the controller is the value after both the propagation and update steps, that is $\hat{\BB a}(t_k) = \hat{\BB a}_{t_k^+}$. During the propagation step in \cref{eq:discrete-propagate-a,eq:discrete-propagate-P} both $\ahat$ and $P$ are regularized. Then, in the update step in \cref{eq:discrete-update-a,eq:discrete-update-P}, $P$ and $\ahat$ are updated according to the gain in \cref{eq:discrete-kalman-gain}. This mirrors a discrete Kalman filter implementation \cite{kalman_new_1960} with the tracking error term added in the update step. The discrete Kalman filter exactly integrates the continuous time Kalman filter when the prediction error $\BB e$, tracking error $\BB s$, and learned basis functions $\phi$ are constant between time steps ensuring the positive definiteness of $P$. 
\begin{align}
    \label{eq:discrete-propagate-a}
    \hat{\BB a}_{t_k^-} &= \underbrace{(1 - \lambda \Delta t_k)}_{\text{damping}} \hat{\BB a}_{t_{k-1}^+}
    \\
    \label{eq:discrete-propagate-P}
    P_{t_k^-} &= (1 -  \lambda \Delta t_k)^2 P_{t_{k-1}^+} + Q \Delta t_k 
    \\
    \label{eq:discrete-kalman-gain}
    K_{t_k} &= P_{t_k^-}  \phi_{t_k}^\top \left(\phi_{t_k} P_{t_k^-} \phi_{t_k}^\top + R \Delta t_k\right)^{-1} 
    \\
    \label{eq:discrete-update-a}
    \hat{\BB a}_{t_k^+} &= \hat{\BB a}_{t_k^-} - \underbrace{K_{t_k}\left( \phi_{t_k}\hat{\BB a}_{t_k^-} - \BB y_{t_k}\right)}_{\text{prediction error adaptation}} - \underbrace{P_{t_k^-} \phi_{t_k}^\top \BB s_{t_k}}_{\text{tracking error adaptation}} 
    \\
    \label{eq:discrete-update-P}
    P_{t_k^+} &= \left(I - K_{t_k} \phi_{t_k}\right) P_{t_k^-} \left(I - K_{t_k} \phi_{t_k}\right)^\top + K_{t_k} R \Delta t_k K_{t_k}^\top 
\end{align}

\subsection{Stability and Robustness Formal Guarantees and Proof}
\label{sec:supp-stability}

We divide the proof of \Cref{eq:exponential-bound} into two steps. First, in \cref{thm:s-stability}, we show that the combined composite velocity tracking error and adaptation error, $\|[\BB s;\tilde{\BB a}]\|$, exponentially converges to a bounded error ball. This implies the exponential convergence of $\BB s$. Then in \Cref{cor:q-bound} we show that when $\BB s$ is exponentially bounded, $\tilde{\BB q}$ is also exponentially bounded. Combining the exponential bound from \cref{thm:s-stability} and the ultimate bound from \cref{cor:q-bound} proves \cref{thm:stability-proof}.

Before discussing the main proof, let us consider the robustness properties of the feedback controller without considering any specific adaptation law. Taking the dynamics \cref{eq:open-loop-dynamics}, control law \cref{eq:control-law-our}, the composite velocity error definition \cref{eq:s-dynamics}, and the parameter estimation error $\tilde{\BB a}=\hat{\BB a}-\BB a$, we find
\begin{equation}
    \label{eq:s-closed-loop}
    M\dot{\BB s}+(C+K)\BB s=-\phi \tilde{\BB a}+\BB d
\end{equation}
We can use the Lyapunov function $\mathcal{V}=\BB s^\top M \BB s$ under the assumption of bounded $\tilde{\BB a}$ to show that
\begin{align}
  \lim_{t\rightarrow\infty}\left\| \BB s \right\| &\leq \frac{\sup_t\|\BB d-\phi\tilde{\BB a}\| \lammax(M)}{\lammin(K) \lammin(M)}
\end{align}
Taking this results alone, one might expect that any online estimator or learning algorithm will lead to good performance. 
However, the boundedness of $\tilde{\BB a}$ is not guaranteed; Slotine and Li discuss this topic thoroughly \cite{slotine1991applied}. In the full proof below, we show the stability and robustness of the \nf~adaptation algorithm.

First, \revC{we introduce the parameter measurement noise $\bar{\epsilon}$, where $\bar{\epsilon}=\BB y - \phi \BB a$. Thus, $\bar{\epsilon}=\epsilon+\BB d$ and $\|\bar{\epsilon}\|\leq \|\epsilon\| + \|\BB d\|$ by the triangle inequality. 
Using the above closed loop dynamics \cref{eq:s-closed-loop}, the parameter estimation error $\tilde{\BB a}$, and the adaptation law \cref{eq:adaptation-law-a,eq:adaptation-law-P}, the combined velocity and parameter-error closed-loop dynamics are given by} \begin{align}
    \begin{bmatrix}
        M & 0 \\
        0 & P^{-1} 
    \end{bmatrix}
    \begin{bmatrix}
        \dot{\BB s} \\ \dot {\tilde{\BB a}}
    \end{bmatrix} 
    + 
    \begin{bmatrix}
        C + K & \phi \\
        -\phi^T & \phi^\top R\inv \phi + \lambda P \inv \\ 
    \end{bmatrix}
    \begin{bmatrix}
        \BB s \\ \tilde{\BB a}
    \end{bmatrix} 
    =
    \begin{bmatrix}
        \BB d \\
        \phi^\top R\inv \bar{\epsilon} - P\inv \lambda \BB a - P\inv \dot{\BB a}
    \end{bmatrix}
    \label{eq:stacked-dynamics}
    \\
    \frac{d}{dt}\left(P\inv\right) = - P\inv \dot P P\inv = P\inv \left(2 \lambda P - Q + P \phi ^\top R^{-1} \phi P\right) P\inv
    \label{eq:Pinv-dynamics}
\end{align}

For our stability proof, we rely on the fact that $P\inv$ is both uniformly positive definite and uniformly bounded, that is, there exists some positive definite, constant matrices $A$ and $B$ such that $A \succeq P\inv \succeq B$.
Dieci and Eirola \cite{dieci_positive_1994} show the slightly weaker result that that $P$ is positive definite and finite when $\phi$ is bounded under the looser assumption $Q\succeq0$. Following the proof from \cite{dieci_positive_1994} with the additional assumption that $Q$ is uniformly positive definite, one can show the uniform definiteness and uniform boundedness of $P$. Hence, $P\inv$ is also uniformly positive definite and uniformly bounded.

\begin{theorem}
\label{thm:s-stability}
Given dynamics that evolve according to \cref{eq:stacked-dynamics,eq:Pinv-dynamics}, uniform positive definiteness and uniform boundedness of $P\inv$, the norm of $\stackedstate$ exponentially converges to the bound given in \cref{eq:exponential-bound-long} with rate $\alpha$.
\begin{align}
    \label{eq:exponential-bound-long}
    \lim_{t\rightarrow\infty}\left\|\stackedstate\right\| &\leq \frac{1}{\alpha \lammin(\metric)}
    \left( \sup_t\|\BB d\| + \sup_t(\|\phi^\top R\inv \bar{\epsilon}\|) + \lammax(P\inv)\sup_t(\|\lambda \BB a + \dot{\BB a}\|) \right)
\end{align}
where $\alpha$ and $\metric$ are functions of $\phi, R, Q, K, M$ and $\lambda$, and $\lammin(\cdot)$ and $\lammax(\cdot)$ are the minimum and maximum eigenvalues of $(\cdot)$ over time, respectively. Given \cref{cor:q-bound} and \cref{eq:exponential-bound-long}, the bound in \cref{eq:exponential-bound} is proven. Note $\lammax(P^{-1})=1/\lammin(P)$ and a sufficiently large value of $\lammin(P)$ will make the RHS of \cref{eq:exponential-bound-long} small.
\end{theorem}
\begin{proof}

Now consider the Lyapunov function $\lyap$ given by
\begin{align}
    \lyap &= \stackedstate^\top
        \begin{bmatrix}
            M & 0 \\
            0 & P\inv
        \end{bmatrix}
        \stackedstate
\end{align}
This Lyapunov function has the derivative
\begin{align}
    \dot{\lyap} &= 2 \stackedstate^\top 
        \begin{bmatrix}
            M & 0 \\
            0 & P\inv
        \end{bmatrix}
        \stackedstatedot
        +
        \stackedstate^\top 
        \begin{bmatrix}
            \dot M & 0 \\
            0 & \frac{d}{dt}\left(P\inv \right) 
        \end{bmatrix}
        \stackedstate
    \\ &=
        - 2 \stackedstate^\top
        \begin{bmatrix}
            C + K & \phi \\
            -\phi^T & \phi^\top R\inv \phi + \lambda P \inv \\ 
        \end{bmatrix}
        \begin{bmatrix}
            \BB s \\ \tilde{\BB a}
        \end{bmatrix}
        + 2 \stackedstate^\top
        \begin{bmatrix}
            \BB d \\
            \phi^\top R\inv \bar{\epsilon} - P\inv \lambda \BB a - P\inv \dot{\BB a}
        \end{bmatrix}
        +
        \\&\quad
        \stackedstate^\top 
        \begin{bmatrix}
            \dot M & 0 \\
            0 & \frac{d}{dt}\left(P\inv \right)
        \end{bmatrix}
        \stackedstate
    \\ &=
        - 2 \stackedstate^\top
        \begin{bmatrix}
            K & \phi \\
            -\phi^T & \phi^\top R\inv \phi + \lambda P \inv \\ 
        \end{bmatrix}
        \begin{bmatrix}
            \BB s \\ \tilde{\BB a}
        \end{bmatrix}
        + 2 \stackedstate^\top
        \begin{bmatrix}
            \BB d \\
            \phi^\top R\inv \bar{\epsilon} - P\inv \lambda \BB a - P\inv \dot{\BB a}
        \end{bmatrix}
        \\&\quad
        +
        \stackedstate^\top 
        \begin{bmatrix}
            0 & 0 \\
            0 & 2 \lambda P\inv - P\inv Q P\inv + \phi ^\top R^{-1} \phi
        \end{bmatrix}
        \stackedstate
    \\ &=
        - \stackedstate^\top
        \begin{bmatrix}
            2 K & 0 \\
            0 & \phi^\top R\inv \phi + P\inv Q P\inv\\
        \end{bmatrix}
        \begin{bmatrix}
            \BB s \\ \tilde{\BB a}
        \end{bmatrix}
        + 2 \stackedstate^\top
        \begin{bmatrix}
            \BB d \\
            \phi^\top R\inv \bar{\epsilon} - P\inv \lambda \BB a - P\inv \dot{\BB a}
        \end{bmatrix}
\end{align}
where we used the fact $\dot{M}-2C$ is skew-symmetric. As $K$, $ P\inv Q P\inv $, $M$, and $P\inv$ are all uniformly positive definite and uniformly bounded, and $\phi^\top R\inv \phi$ is positive semidefinite, there exists some $\alpha > 0$ such that 
\begin{align}
    \label{eq:convergence-rate}
        - \begin{bmatrix}
            2 K & 0 \\
            0 & \phi^\top R\inv \phi + P\inv Q P\inv\\
        \end{bmatrix}
        &\preceq - 2 \alpha
        \begin{bmatrix}
            M & 0 \\
            0 & P\inv
        \end{bmatrix}
\end{align}
for all $t$.

Define an upper bound for the disturbance term $D$ as
\begin{align}
    D = \sup_t \left\|
        \begin{bmatrix}
            \BB d \\
            \phi^\top R\inv \bar{\epsilon} - P\inv \lambda \BB a - P\inv \dot{\BB a}
        \end{bmatrix}
        \right\|
\end{align}
and define the function $\metric$,
\begin{align}
    \metric &= 
        \begin{bmatrix}
            M & 0 \\
            0 & P\inv
        \end{bmatrix}
\end{align}
By \cref{eq:convergence-rate}, the Cauchy-Schwartz inequality, and the definition of the minimum eigenvalue, we have the following inequality for $\dot\lyap$:
\begin{align}
    \dot{\lyap} &\leq 
        - 2 \alpha \lyap
        + 2 \sqrt{\frac{\lyap}{\lammin(\metric)}} D
\end{align}
Consider the related systems, $\mathcal{W}$ where $\mathcal{W} = \sqrt{\lyap}$, $2\dot{\mathcal{W}} \mathcal{W} = \dot\lyap$, and the following three equations hold
\begin{align}
    2\dot{\mathcal{W}}\mathcal{W} &\leq - 2 \alpha \mathcal{W}^2 +  \frac{2 D \mathcal{W}}{\sqrt{\lammin(\metric)}}
    \\
    \dot{\mathcal{W}} &\leq - \alpha \mathcal{W} +  \frac{D}{\sqrt{\lammin(\metric)}}
\end{align}
By the Comparison Lemma \cite{khalil_nonlinear_2002},
\begin{align}
    \sqrt{\lyap} = \mathcal{W} \leq \expo{-\alpha t} \left(\mathcal{W}(0) - \frac{D}{\alpha \sqrt{\lammin(\metric)}} \right) + \frac{D}{\alpha \sqrt{\lammin(\metric)}}
\end{align}
and the stacked state exponentially converges to the ball
\begin{align}
    \lim_{t\rightarrow\infty}\left\| \stackedstate \right\| \leq \frac{D}{\alpha \lammin(\metric)} 
\end{align}
This completes the proof.
\end{proof}

Next, we present a corollary which shows the exponential convergence of $\tilde{\BB q}$ when $\BB s$ is exponentially stable.
\begin{corollary}
\label{cor:q-bound}

If $\|\BB s(t)\|\leq A \exp(-\alpha t) + B / \alpha$ for some constants $A$, $B$, and $\alpha$, and $ \BB s = \dot{\tilde{\BB q}} + \Lambda \tilde{\BB q}$, then 
\begin{align}
\label{eq:q-bound-messy}
\|\tilde{\BB q}\| \leq \expo{-\lammin(\Lambda)t}\|\tilde{\BB q}(0)\| + \int_0^t \expo{-\lammin(\Lambda)(t-\tau)} A \expo{- \alpha \tau}\text{\emph{d}}\tau +
\int_0^t\expo{-\lammin(\Lambda)(t-\tau)}\frac{B}{\alpha} \text{\emph{d}}\tau
\end{align}
thus $\|\tilde{\BB q}\|$ exponentially approaches the bound
\begin{align}
    \label{eq:q-cor-bound}
    \lim_{t\rightarrow\infty}\|\tilde{\BB q}\| &\leq \frac{B}{\alpha \lammin(\Lambda)}
\end{align}
\end{corollary}

\begin{proof}
From the Comparison Lemma \cite{khalil_nonlinear_2002}, we can easily show \cref{eq:q-bound-messy}. This can be further reduced as follows.
\begin{align}
    \|\tilde{\BB q}\| 
    &\leq \expo{-\lammin(\Lambda)t}\|\tilde{\BB q}(0)\| + A\expo{-\lammin(\Lambda)t}\int_0^t \expo{(\lammin(\Lambda)- \alpha) \tau}\text{d}\tau +
    \int_0^t\expo{-\lammin(\Lambda)(t-\tau)}\frac{B}{\alpha} \text{d}\tau\\
    &
    \leq \expo{-\lammin(\Lambda)t}\|\tilde{\BB q}(0)\| + A\frac{\expo{-\alpha t}-\expo{-\lammin(\Lambda)t}}{\lammin(\Lambda)-\alpha}+ \frac{B \left(1-\expo{-\lammin(\Lambda)t}\right)}{\alpha \lammin(\Lambda)}
\end{align}
Taking the limit, we arrive at \cref{eq:q-cor-bound}

\end{proof}

With the following corollary, we will justify that $\alpha$ is strictly positive even when $\phi\equiv0$, and thus the adaptive control algorithm guarantees robustness even in the absence of persistent excitation or with ineffective learning.
In practice we expect some measurement information about all the elements of $\BB a$, that is, we expect a non-zero $\phi$. 
\begin{corollary}
If $\phi\equiv0$, then the bound in \cref{eq:exponential-bound-long} can be simplified to
\begin{align}
    \label{eq:worst-case-bound}
    \lim_{t\rightarrow\infty}\left\| \stackedstate \right\| &\leq \frac{
        \sup\|\BB d\| + 
        \lammax(P\inv)\sup(\|\lambda \BB a + \dot{\BB a}\|)
    }{
        \min\left(\lambda, \lammin(K)/\lammax(M)\right)\lammin(\metric)
    }
\end{align}
\end{corollary}
\begin{proof}
Assuming $\phi\equiv0$ immediately leads to $\alpha$ of
\begin{align}
    \alpha &= \min\left(\frac{1}{2} \lammin( P\inv Q ), \frac{\lammin(K)}{\lammax(M)}\right)
\end{align}
$\phi\equiv0$ also simplifies the $\dot{P}$ equation to a stable first-order differential matrix equation. By integrating this simplified $\dot{P}$ equation, we can show $P$ exponentially converges to the value $P=\frac{Q}{2\lambda}$. This leads to bound in \cref{eq:worst-case-bound}.

\end{proof}

We now introduce another corollary for the \nfconstant, when $\phi=I$. In this case, the regularization term is not needed, as it is intended to regularize the linear coefficient estimate in the absence of persistent excitation, so we set $\lambda=0$. This corollary also shows that \nfconstant~is sufficient for perfect tracking control when $\BB f$ is constant; though in this case, even \rev{the nonlinear baseline controller with} integral control will converge to perfect tracking. In practice for quadrotors, we only expect $\BB f$ to be constant when the drone air-velocity is constant, such as in hover or steady level flight with constant wind velocity.
\begin{corollary}
If $\phi\equiv I$, $Q=q I$, $R=r I$, $\lambda=0$, and $P(0)=p_0 I$ is diagonal, where $q$, $r$ and $p_0$ are strictly positive scalar constants, then the bound in \cref{eq:exponential-bound-long} can be simplified to
\begin{align}
    \lim_{t\rightarrow\infty}\left\| \stackedstate \right\| &\leq \frac{ \left(\left(1 + r^{-1}\right) \sup_t\|\BB f - \BB a\| \rev{+ \epsilon/r}\right) \lammax(M)}{\lammin(K) \lammin(\metric)}
\end{align}
\end{corollary}
\begin{proof}
Under these assumptions, the matrix differential equation for $P$ is reduced to the scalar differential equation
\begin{align}
    \frac{dp}{dt} &= q - p^2/r
\end{align}
where $P(t)=p(t) I$. This equation can be integrated to find that $p$ exponentially converges to $p = \sqrt{q r}$. Then by \cref{eq:convergence-rate}, $\alpha \leq \sqrt{q/r}$ and $\alpha \leq \lammin(K)/\lammax(M)$. If we choose $q$ and $r$ such that $\sqrt{q/r} = \lammin(K)/\lammax(M)$, then we can take $\alpha = \lammin(K)/\lammax(M)$. Then, the error bound reduces to 
\begin{align}
    \lim_{t\rightarrow\infty}\left\| \stackedstate \right\| &\leq \frac{D \lammax(M)}{\lammin(K) \lammin(\metric)}
\end{align}
Take $\BB a$ as a constant. Then $\dot{\BB a}=0$, $\BB d=\BB f - \BB a$, and $D$ is bounded by
\begin{align}
    D &\leq  \left(1 + r^{-1}\right) \sup_t\|\BB f - \BB a\| + \epsilon/r
\end{align}

\end{proof}

\subsection{Gain Tuning}
\label{sec:supp-gain-tuning}

The attitude controller was tuned following the method in \cite{noauthor_multicopter_nodate}.
\rev{The gains for all of the position controllers tested were tuned on a step input of \SI{1}{m} in the x-direction. The proportional (P) and derivative (D) gains were tuned using the baseline nonlinear controller for good rise time with minimal overshoot or oscillations. The same P and D gains were used across all methods.} 

\rev{The integral and adaptation gains were tuned separately for each method. In each case, the gains were increased to minimize response time until we observed having large overshoot, noticeably jittery, or oscillatory behavior. For $\LOne$ and INDI this gave a first-order filters with a cutoff frequency of \SI{5}{Hz}. For each of the \nf~methods, we used $R=r I$ and $Q=q I$, where $r$ and $q$ are a scalar values. The tuning method gave an $R$ gains similar to the measurement noise of the residual force, a $Q$ values on the order of $0.1$, and $\lambda$ values of $0.01$.}

\subsection{Force Prediction Performance}
\label{supp:prediction-performance}

The section discusses fig.~\ref{fig:force-prediction}, which is useful for understanding why learning improves force prediction (which in turn improves control).

\begin{figure}
    \centering
    \includegraphics[width=\textwidth]{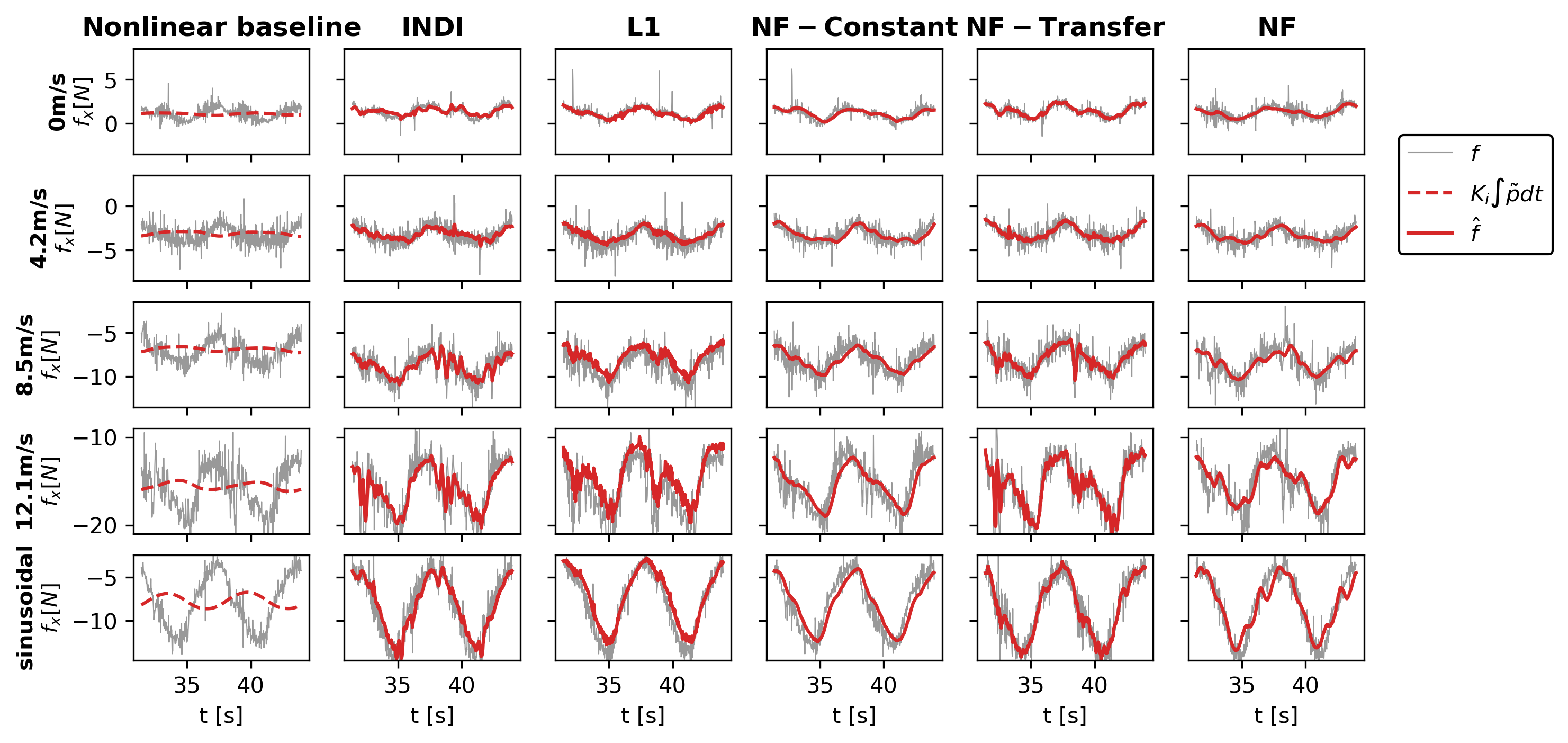}
    \includegraphics[width=\textwidth]{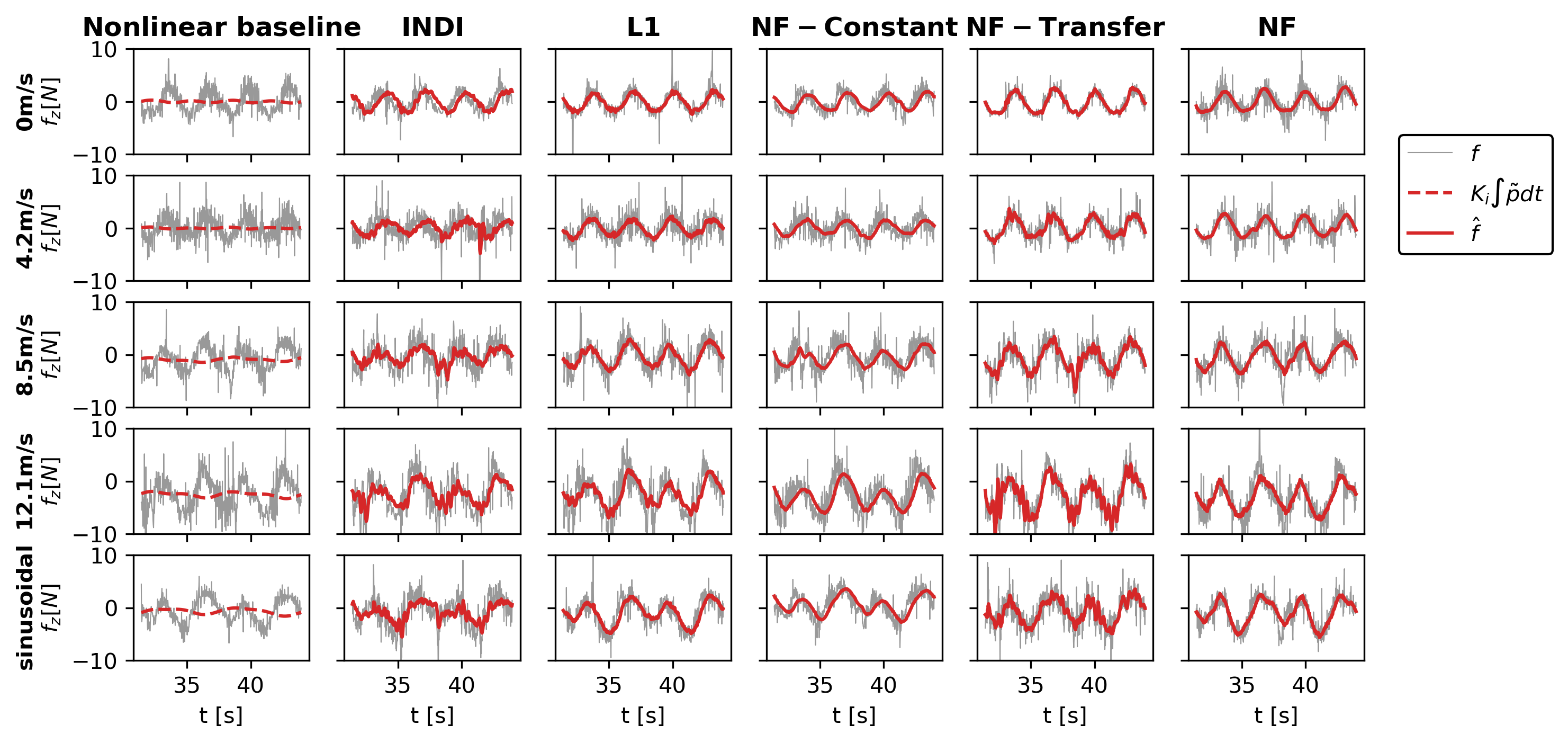}
    \caption{
        \textbf{Measured residual force versus adaptive control augmentation, $\hat{\BB f}$.} Wind-effect x- and z-axis force prediction for different methods, $\hat{\BB f}$ and $K_i \int\tilde{\BB p}\text{d}t$, compared with the online residual force measurement, $\BB f$. The integral term in the nonlinear baseline method and the $\hat{\BB f}$ term in the adaptive control methods, including the \nf~methods, all act to compensate for the measured residual force. INDI, L1, and \nfconstant~estimate the residual force with sub-second lag, however adjusting the gains to decrease the lag increases noise amplification. \nf~and \nftransfer~have reduced the lag in estimating the residual force but have some model mismatch, especially at higher wind speeds.}
    \label{fig:force-prediction}
\end{figure}

For the nonlinear baseline method, the integral (I) term compensates for the average wind effect, as seen in fig.~\ref{fig:force-prediction}. Thus, the UAV trajcetory remains roughly centered on the desired trajectory for all wind conditions, as seen in fig.~\ref{fig:trajectory}. The relative velocity of the drone changes too quickly for the integral-action to compensate for the changes in the wind effect. Although increasing the I gain would allow the integral control to react more quickly, a large I gain can also lead to overshoot and instability, thus the gain is effectively limited by the combined stability of the P, D, and I gains. 

Next, consider the two SOTA baseline methods, INDI and $\LOne$, along with the non-learning version of our method, \nfconstant. These methods represent different adaptive control approaches that assume no prior model for the residual dynamics. Instead, each of these methods effectively outputs a filtered version of the measured residual force and the controller compensates for this adapted term. In fig.~\ref{fig:force-prediction}, we observe that each of these methods has a slight lag behind the measured residual force, in grey. This lag is reduced by increasing the adaptation gain, however, increasing the adaptation gain leads to noise amplification. Thus, these \emph{reactive} approaches are limited by some more inherent system properties, like measurement noise.

Finally, consider the two learning versions of our method, \nf~and \nftransfer. These methods use a learned model in the adaptive control algorithm. Thus, once the linear parameters have adapted to the current wind condition, the model can predict future aerodynamic effects with minimal changes to the coefficients. As we extrapolate to higher wind speeds and time-varying conditions, some model mismatch occurs and is manifested as discrepancies between the predicted force, $\hat{\BB f}$, and the measured force, $\BB f$, as seen in fig.~\ref{fig:force-prediction}. Thus, our learning based control is limited by the learning representation error. This matches the conclusion drawn in our theoretical analysis, where tracking error scales linearly with representation error.


\subsection{Localization Error Analysis}
\label{sec:supp-localization-error}

We estimate the root mean squared position localization precision to be about \SI{1}{cm}. This is based on a comparison of our two different localization data sources. The first is the OptiTrack motion capture system, which uses several infrared motion tracking cameras and reflective markers on the drone to produce a delayed measurement the position and orientation of the vehicle. The PX4 flight controller runs an onboard extended Kalman filter (EKF) to fuse the OptiTrack measurements with onboard inertial measurment unit (IMU) measurements to produce position, orientation, velocity, and angular rate estimates. In offline analysis, we correct for the delay of the OptiTrack system, and compare the position outputs of the OptiTrack system and the EKF. Typical results are shown in fig.~\ref{fig:localization-good}. The fixed offset between the measurements occurs because the OptiTrack system tracks the centroid of the reflective markers, where the EKF tracks the center of mass of the vehicle. Although the EKF must internally correct for this offset, we do not need to do so in our offline analysis because the offset is fixed. Thus, the mean distance between the the OptiTrack position and the EKF position corresponds to the distance between the center of mass and the center of vision, and the standard deviation of that distance is the root-mean-square error of the error between the two estimates. Averaged over all of the data from experiments in this paper, we see that the standard deviation is \SI{1.0}{cm}. Thus, we estimate that the localization precision has a standard deviation of about \SI{1.0}{cm}.

\begin{figure}
    \centering
    \includegraphics[width=\linewidth]{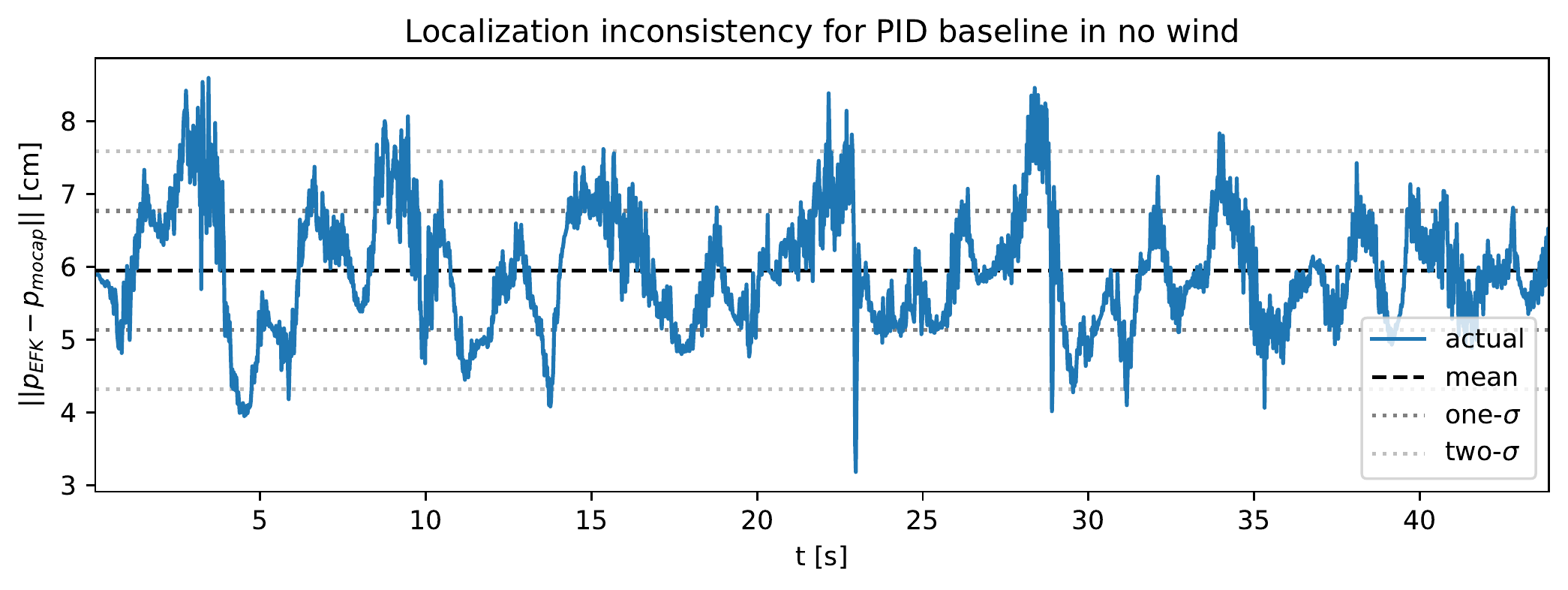}
    \caption{\textbf{Localization inconsistency} Typical difference between the OptiTrack motion capture position measurement, $\BB p_\text{mocap}$, and the EKF position estimate, $\BB p_\text{EKF}$, corrected for the Optitrack delay. The mean difference corresponds to a constant offset between the center of mass, which the EKF tracks, and the centroid of reflective markers, which the OptiTrack measures. The standard deviation corresponds to the root-mean-square error between the two measurements.}
    \label{fig:localization-good}
\end{figure}

\end{document}